\documentclass{article}

     \usepackage[nonatbib, final]{neurips_2020}

\usepackage[utf8]{inputenc} 
\usepackage[T1]{fontenc}    
\usepackage{hyperref}       
\usepackage{url}            
\usepackage{booktabs}       
\usepackage{amsfonts}       
\usepackage{nicefrac}       
\usepackage{microtype}      

\usepackage{amsmath}
\usepackage{amsthm}
\usepackage{amssymb}
\usepackage{bbm}
\usepackage{mathtools}
\usepackage{mathrsfs}
\usepackage{dsfont}
\usepackage{theoremref}
\mathtoolsset{showonlyrefs}

\newtheorem{thm}{Theorem}[]

\newtheorem{lemma}{Lemma}[]
\newtheorem{prop}{Property}[]
\newtheorem{ass}{Assumption}[]

\newtheorem{reqm}{Requirement}

\usepackage{courier} 

\usepackage{lipsum} 

\usepackage[space]{grffile} 

\usepackage[numbers]{natbib}

\usepackage{hyperref}

\usepackage{graphicx}
\usepackage{wrapfig}
\usepackage{xcolor}
\definecolor{dark-red}{rgb}{0.4,0.15,0.15}
\definecolor{dark-blue}{rgb}{0,0,0.7}
\hypersetup{
    colorlinks, linkcolor={dark-blue},
    citecolor={dark-blue}, urlcolor={dark-blue}
}

\DeclareMathOperator*{\argmax}{arg\,max}

\usepackage{algpseudocode}
\usepackage[vlined,linesnumbered,ruled,algo2e]{algorithm2e}
\SetKwProg{Fn}{Function}{}{}
\let\oldnl\nl
\newcommand{\nonl}{\renewcommand{\nl}{\let\nl\oldnl}}

\usepackage{multicol}
\usepackage{enumitem}
\usepackage[utf8]{inputenc} 
\usepackage[T1]{fontenc}    
\usepackage{hyperref}       
\usepackage{url}            
\usepackage{booktabs}       
\usepackage{nicefrac}       

\newcommand{\myparagraph}[1]{\vspace{-0.2cm}\par{\hspace{0.1cm}}\\ \noindent \textbf{#1} \hspace{0.2cm}}

\title{Towards Safe Policy Improvement  for\\  Non-Stationary MDPs}

\author{%
Yash Chandak\\
  University of Massachusetts\\
  \texttt{ychandak@cs.umass.edu} \\
   \And
   Scott M. Jordan \\
   University of Massachusetts\\
  \texttt{sjordan@cs.umass.edu} \\
   \And
   Georgios Theocharous \\
   Adobe Research\\
  \texttt{theochar@adobe.com} \\
   \And
   Martha White \\
   University of Alberta \& Amii\\
  \texttt{whitem@alberta.ca} \\
   \And
   Philip S. Thomas \\
   University of Massachusetts\\
  \texttt{pthomas@cs.umass.edu} \\
}

\begin{document}

\maketitle

\begin{abstract}
Many real-world sequential decision-making problems involve critical systems with financial risks and human-life risks.
While several works in the past have proposed methods that are \textit{safe} for deployment, they assume that the underlying problem is \textit{stationary}.
However, many real-world problems of interest exhibit non-stationarity, and when stakes are high, the cost associated with a false stationarity assumption may be unacceptable.
%
%
We take the first steps towards ensuring safety, with high confidence, for smoothly-varying non-stationary decision problems.
Our proposed method extends a type of safe algorithm, called a \emph{Seldonian algorithm}, through a synthesis of model-free reinforcement learning with time-series analysis.
Safety is ensured using sequential hypothesis testing of a policy's \textit{forecasted} performance, and confidence intervals are obtained using \textit{wild bootstrap}.
\end{abstract}
\vspace{-5pt}
\section{Introduction}
Reinforcement learning (RL) methods have been applied to real-world sequential decision-making problems such as diabetes management \cite{bastani2014model}, sepsis treatment \cite{saria2018individualized}, and budget constrained bidding \citep{wu2018budget}.
For such real-world applications, safety guarantees are critical to mitigate serious risks in terms of both human-life and monetary assets. More concretely, here, by \textit{safety} we mean that any update to a system should not reduce the performance of an existing system (e.g., a doctor's initially prescribed treatment).
A further complication is that these applications are non-stationary, violating the foundational assumption \cite{sutton2018reinforcement} of \textit{stationarity} in most RL algorithms.
This raises the main question we aim to address: \textit{How can we build sequential decision-making systems that provide safety guarantees for problems with non-stationarities?}

Conventionally, RL algorithms designed to ensure safety \citep{pirotta2013safe,garcia2015comprehensive, thomas2015safe,  zhang2016query, laroche2017safe, chow2018lyapunov} model the environment as a Markov decision process (MDP), and rely upon the \textit{stationarity assumption} made by MDPs \cite{sutton2018reinforcement}.
That is, MDPs assume that a decision made by an \textit{agent} always results in the same (distribution of) consequence(s) when the environment is in a given state.
Consequently, safety is only ensured by prior methods when this assumption holds, which is rare in real-world problems.

While some works for lifelong-reinforcement learning \cite{brunskill2014pac, abel2018policy, chandak2020lifelong, chandak2020future} or meta-reinforcement learning \cite{al2017continuous,xie2020deep} do aim to address the problem of non-stationarity, they do not provide any safety guarantees.
Perhaps the work most closely related to ours is by \citet{ammar2015safe}, which aims to find a policy that satisfies a safety constraint in the lifelong-learning setting.
They use a follow-the-regularized-leader (FTRL) \cite{shalev2012online} approach to first perform an unconstrained maximization over the \textit{average} performance over all the trajectories collected in the past, and then project the resulting solution onto a safe set.
However, as shown by \citet{chandak2020future}, FTRL based methods can suffer from a significant performance lag in non-stationary environments.
Further, the parameter projection requires \textit{a priori} knowledge of the set of safe policy parameters, which might be infeasible to obtain for many problems, especially when the constraint is to improve performance over an existing policy or when the safe set is non-convex (e.g., when using policies parameterized using neural networks).
Additionally, the method proposed by \citet{chandak2020future} for policy improvement does not provide safety guarantees, and thus it would be irresponsible to apply it to safety-critical problems.
%
%
%

%
%
\textbf{Contributions:} In this work, we formalize the \textit{safe policy improvement} problem for a more realistic non-stationary setting and provide an algorithm for addressing it. 
    Additionally, a user-controllable knob is provided to set the desired \textit{confidence level}: the maximum admissible probability that a worse policy will be deployed.
The proposed method relies only on estimates of future performance, with associated confidence intervals. It does not require building a model of a non-stationary MDP (NS-MDP), and so it is applicable to a broader class of problems, as modeling an NS-MDP can often be prohibitively difficult.
In Figure \ref{fig:idea}, we provide an illustration of the proposed approach for ensuring safe policy improvement for NS-MDPs.

\textbf{Limitations: } 
The method that we propose is limited to settings where both (a) non-stationarity is governed by an exogenous process (that is, past actions do not impact the underlying non-stationarity), and (b) the performance of every policy changes smoothly over time and has no discontinuities (abrupt breaks or jumps).
Further, the use importance sampling makes our method prone to high variance.

\begin{figure}[t]
    \centering
    \includegraphics[width=0.9\textwidth]{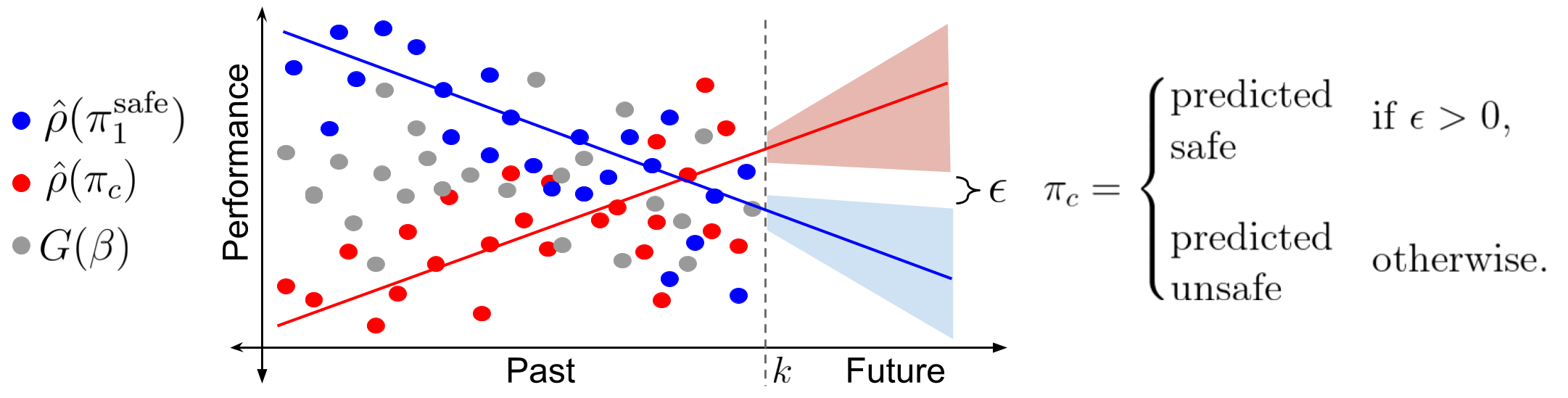}
    \caption{An illustration of the proposed idea where \textit{safety} is defined to ensure that the future performance of a proposed policy $\pi_c$ is never worse than that of an existing, known, safe policy $\pi^\text{safe}$.
    The gray dots correspond to the returns, $ G(\beta)$, observed for a policy $\beta$. 
    The red and the blue dots correspond to the counterfactual estimates, $\hat \rho(\pi_c)$ and $\hat \rho(\pi^\text{safe})$, for performance of $\pi_c$ and $\pi^\text{safe}$, respectively.
    The shaded regions correspond to the uncertainty in future performance obtained by analysing the trend of the counterfactual estimates for past performances. 
    %
    }
    \label{fig:idea}
\end{figure}

\section{Notation}
 
We represent an NS-MDP as a \textit{stochastic sequence}, $\{ M_i\}_{i=1}^\infty$, of stationary MDPs $M_i \in \mathcal M$, where $\mathcal M$ is the set of all stationary MDPs.
Each $M_i$ is a tuple $(\mathcal S, \mathcal A, \mathcal P_i, \mathcal R_i, \gamma, d^0)$, where $\mathcal S$ is the set of possible states, $\mathcal A$ is the set of actions, $\gamma \in [0,1)$ is the \textit{discounting factor} \cite{sutton2018reinforcement}, $d^0$ is the start-state distribution,  $\mathcal R_i: \mathcal S \times \mathcal A \rightarrow \Delta(\mathbb{R})$ is the reward distribution, and $\mathcal P_i: \mathcal S \times \mathcal A \rightarrow \Delta(\mathcal S)$ is the transition function, where $\Delta$ denotes a conditional distribution.
For all $M_i \in \mathcal M$,  we assume that $\mathcal S, \mathcal A, \gamma, \text{and} \,\, d^0$ remain fixed.
We represent a policy as  $\pi: \mathcal S \rightarrow \Delta(\mathcal A)$.

Let $s_\infty$ be a \textit{terminal absorbing state} \cite{sutton2018reinforcement} and an \textit{episode} be a sequence of interactions with a given MDP, which enters $s_\infty$ within $T$ time steps.
In general, we will use subscripts to denote the episode number and superscripts to denote the time step within an episode.
That is, $S^t_i, A^t_i, $ and $R^t_i$ are the random variables corresponding to the state, action, and reward at time step $t$ in episode $i$. 
Let a trajectory for episode $i$ generated using a policy (also known as a \textit{behavior policy}) $\beta_i$ be $H_i \coloneqq \{S^j_i, A^j_i, \beta_i(A^j_i|S^j_i), R^j_i\}_{j=0}^{\infty}$, where
 $\forall (i, t), \,\, R_i^t \in  [-R_\text{max}, R_\text{max}]$.
Let a \textit{return} of $\pi$ for any $m \in \mathcal M$ be $G(\pi, m) \coloneqq \sum_{t=0}^\infty \gamma^t R^t$ and the \textit{expected return} $\rho(\pi, m) \coloneqq \mathbb{E}[G(\pi, m)]$.
%
With a slight overload of notation, let the \textit{ performance} of $\pi$ for episode $i$ be $\rho(\pi, i) \coloneqq \mathbb{E}[\rho(\pi, M_i)]$.
We will use $k$ to denote the most recently finished episode, such that episode numbers $[1, k]$ are in the past and episode numbers $(k, \infty]$ are in the future.

\section{The Safe Policy Improvement Problem for Non-stationary MDPs}

In this section, we formally present the problem statement, discuss the difficulty of this problem, and introduce a smoothness assumption that we leverage to make the problem tractable.

\textbf{Problem Statement: } 
Let $\mathcal D \coloneqq \{(i, H_i) : i \in [1, k] \}$ be a random variable denoting a set of trajectories observed in the past and let $\texttt{alg}$ be an algorithm that takes $\mathcal D$ as input and returns a policy $\pi$.
Let $\pi^\text{safe}$ be a known safe policy, and let $(1 - \alpha) \in [0,1]$ be a constant selected by a user of \texttt{alg}, which we call the \textit{safety level}.
We aim to create an algorithm \texttt{alg} that ensures with high probability that $\texttt{alg}(\mathcal D)$, the policy proposed by \texttt{alg}, does not perform worse than the existing safe policy $\pi^\text{safe}$ during the \textit{future} episode $k + \delta$. 
That is, we aim to ensure the following \textit{safety guarantee},
\begin{align}
    \Pr \Big(\rho(\texttt{alg}(\mathcal D), {k+\delta}) \geq \rho(\pi^\text{safe}, k + \delta) \Big) \geq 1 - \alpha. \label{eqn:constraint}
\end{align}%

\textbf{Hardness of the Problem:}
While it is desirable to ensure the safety guarantee in \eqref{eqn:constraint}, obtaining a new policy from $\texttt{alg}(\mathcal D)$ that meets the requirement in \eqref{eqn:constraint} might be impossible unless some more regularity assumptions are imposed on the problem. 
To see why, notice that if the environment can change arbitrarily, then there is not much hope of estimating $\rho(\pi, {k+\delta})$ accurately since $\rho(\pi, k + \delta)$ for any $\pi$ could be any value between the extremes of all possible outcomes, regardless of the data collected during episodes $1$ through $k$.

To avoid arbitrary changes, previous works typically require the transition function $\mathcal P_k$ and the reward function $\mathcal R_k$ to be Lipschitz smooth over time \citep{lecarpentier2019non,jagerman2019people,lecarpentier2020lipschitz,Cheung2020drifting}. And in fact, we can provide a bound on the change in performance given such Lipschitz conditions, as we show below in Theorem \ref{thm:lipbound}.
Unfortunately, this bound is quite large: unless the Lipschitz constants are so small that they effectively make the problem stationary, the performance of a policy $\pi$ across consecutive episodes can still fluctuate wildly. Notice that due to the inverse dependency on $(1 - \gamma)^2$, if $\gamma$ is close to one,
then the Lipschitz constant $L$ can be enormous even when $\epsilon_P$ and $\epsilon_R$ are small. In Appendix \ref{apx:proof}, we provide an example of an NS-MDP for which \thref{thm:lipbound} holds with exact equality, illustrating that the bound is tight.
\begin{thm}[Lipschitz smooth performance]
\thlabel{thm:lipbound}
If $\exists \epsilon_P \in \mathbb{R}$ and $\exists \epsilon_R \in \mathbb{R}$ such that for any $M_k$ and $M_{k+1}$,
$
\forall s \in \mathcal S, \forall a \in \mathcal A,\,\,\, \lVert \mathcal P_{k}(\cdot| s, a) -\mathcal  P_{k+1}(\cdot| s, a)\rVert_1 \leq \epsilon_P \,\, \text{and} \,\, | \mathbb{E}[\mathcal R_{k}(s, a)] - \mathbb{E}[\mathcal  R_{k+1}(s, a)]| \leq \epsilon_R
$,
then the performance of any policy $\pi$ is Lipschitz smooth over time, with Lipschitz constant $L \coloneqq  \left (\frac{\gamma R_{\text{max}}}{(1 - \gamma)^2}\epsilon_P + \frac{1}{1 - \gamma}\epsilon_R \right)$.
That is, $\forall k \in \mathbb{N}_{> 0}, \forall \delta \in \mathbb{N}_{> 0}, \,\, \,\, |\rho(\pi, k) - \rho(\pi, {k+\delta})| \leq  L \delta$.
%
\end{thm}

%
All proofs are deferred to the appendix.
%
%

%
%
%
%
%

\textbf{An alternate assumption:}
In many real-world sequential decision-making problems, while there is non-stationarity, the performance of a policy $\pi$  does not fluctuate wildly between consecutive episodes. Examples where performance changes are likely more regular include the effect of a medical treatment on a patient; the usefulness of online recommendations based on the interests of a user; or the quality of a controller as a robot's motor friction or battery capacity degrades. 
Therefore, 
instead of considering smoothness constraints on the transition function $\mathcal P_k$ and the reward function $\mathcal R_k$ like above, we consider more direct smoothness constraints on the performance $\rho(\pi, i)$ of a policy $\pi$. 
Similar assumptions have been considered for analyzing trends for digital marketing \cite{thomas2017predictive} and remain popular among policymakers for designing policies based on forecasting \cite{wieland2013forecasting}.

%

If $\rho(\pi,i)$ changes smoothly with episode $i$, then the performance trend of a given policy $\pi$ can be seen as a \textit{univariate time-series}, i.e., a sequence of \textit{scalar} values corresponding to performances $\{\rho(\pi, i)\}_{i=1}^k$ of $\pi$ during episodes $1$ to $k$.
 Leveraging this observation, 
 %
 we propose modeling the performance trend
using a linear regression model that takes an episode number as the input and provides a performance prediction as the output.
To ensure that a wide variety of trends can be modeled, we use a $d$-dimensional non-linear \textit{basis function} $\phi : \mathbb{N}_{>0} \rightarrow \mathbb{R}^{1 \times d}$.
For example, $\phi$ can be the Fourier basis, which has been known to be useful for modeling a wide variety of trends and is fundamental for time-series analysis  \cite{bloomfield2004fourier}.
We state this formally in the following assumption,
\begin{ass}[Smooth performance] 
\thlabel{ass:smooth}
For every policy $\pi$, there exists a sequence of mean-zero and independent noises $\{\xi_i\}_{i=1}^{k+\delta}$, and $\exists w \in \mathbb{R}^{d \times 1} $, such that, $\forall i \in
[1, k+\delta], \,\,\,  \rho(\pi, M_i) = \phi(i)w + \xi_i$.
\end{ass}
Recall that the stochasticity in $\rho(\pi, M_i)$ is a manifestation of stochasticity in $M_i$, and thus this assumption requires that the performance of $\pi$ during episode $i$ is $\rho(\pi, i) = \mathbb{E}[\rho(\pi, M_i)] = \phi(i)w$.

Assumption 1 is reasonable for several reasons. The first is that the noise assumptions are not restrictive. The distribution of $\xi_i$ does not need to be known and the $\xi_i$ can be non-identically distributed. Additionally, both $w$ and $\{\xi_i\}_{i=1}^{k+\delta}$ can be different for different policies. 
The independence assumption only states that at each time step, the variability in performance due to sampling $M_i$ is independent of the past (i.e., there is no auto-correlated noise).

The strongest requirement is that the performance trend be a linear function of the basis $\phi$; but because $\phi$ is a generic basis, this is satisfied for a large set of problems. Standard methods that make stationarity assumptions correspond to our method with $\phi(s) = [1]$ (fitting a horizontal line). Otherwise, $\phi$ is generic: we might expect that there exist sufficiently rich features (e.g., Fourier basis \citep{bloomfield2004fourier})) for which Assumption 1 is satisfied. In practice, we may not have access to such a basis, but like any time-series forecasting problem, goodness-of-fit tests \cite{chen2003empirical} can be used by practitioners to check whether \thref{ass:smooth} is reasonable before applying our method.

The basis requirement, however, can be a strong condition and could be violated. This assumption is \textit{not} applicable for settings where there are jumps or breaks in the performance trend. For example, performance change is sudden when a robot undergoes physical damage, its sensors are upgraded, or it is presented with a completely new task. The other potential violation is the fact that the basis is a function of time. Since the dimension $d$ of the basis $\phi$ is finite and fixed, but $k$ can increase indefinitely, this assumption implies that performance trends of the policies must exhibit a global structure, such as periodicity. This can be relaxed using auto-regressive methods that are better at adapting to the local structure of any time-series. We discuss this and other potential future research directions in Section \ref{sec:broaderImpactpst}. 

\section{SPIN: Safe Policy Improvement for Non-Stationary Settings }
\label{sec:spin}

\begin{figure}
    \centering
    \includegraphics[width=0.9\textwidth]{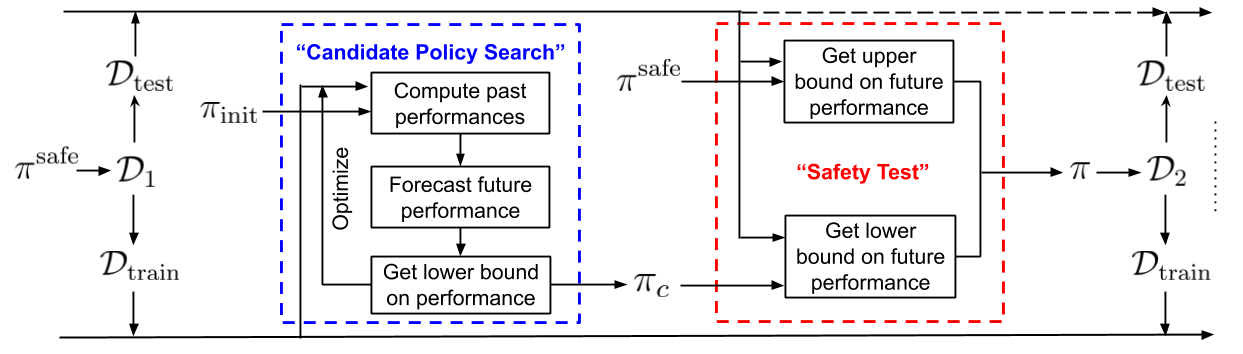}
    \caption{
    %
    The proposed algorithm first partitions the initial data $\mathcal D_1$ into two sets, namely $\mathcal D_\text{train}$ and $\mathcal D_\text{test}$.
    Subsequently, $\mathcal D_\text{train}$ is used to search for a possible \textit{candidate policy} $\pi_c$ that might improve the future performance, and $\mathcal D_\text{test}$ is used to perform a safety test on the proposed candidate policy $\pi_c$.
    The existing safe policy $\pi^\text{safe}$ is only updated if the proposed policy $\pi_c$ passes the safety test.
    \vspace{-10pt}
    }
    \label{fig:Seldonian}
\end{figure}

To ensure safe policy improvement, we adapt the generic template of the Seldonian framework \citep{thomas2019preventing} to the non-stationary setting.
The overall approach consists of continually (1) taking an existing safe policy; (2) finding a candidate policy that has (reasonably high) potential to be a strict improvement on the safe policy; (3) testing if this candidate policy is still safe and is an improvement with high confidence; (4) updating the policy to be the candidate policy only if it passes the test; and (5) gathering more data with the current safe policy to get data for the next candidate policy search.
This procedure consists of four key technical steps: \textit{performance estimation}, \textit{safety test}, \textit{candidate policy search}, and \textit{data-splitting}.
A schematic diagram of the overall procedure is provided in Figure \ref{fig:Seldonian}.

\textbf{Performance Estimation: }
To develop an algorithm that ensures the safety constraint in \eqref{eqn:constraint}, we first require an estimate $\hat \rho(\pi, k+\delta)$ of the future performance $ \rho(\pi, k + \delta)$ and the uncertainty of this estimate, namely a function $\mathscr C$ for obtaining a confidence interval (CI) on future performances. 
Under \thref{ass:smooth}, estimating $\rho(\pi, {k+\delta})$ (the performance  of a policy $\delta$ episodes into the future) can be seen as a \textit{time-series forecasting} problem given the \textit{performance trend} $\{\rho(\pi, i)\}_{i=1}^k$.
We build upon the work by \citet{chandak2020future} to estimate  $\rho(\pi, {k+\delta})$. However, to the best of our knowledge, no method yet exists to obtain $\mathscr C $. A primary contribution of this work is to provide a solution to this technical problem, developed in Section \ref{sec:CI}.

\textbf{Safety Test: } 
To satisfy the required safety constraint in \eqref{eqn:constraint}, an algorithm $\texttt{alg}$ needs to ensure with high-confidence that a given $\pi_c$, which is a \textit{candidate policy} for updating the existing safe policy $\pi^\text{safe}$, will have a higher future performance than that of $\pi^\text{safe}$.
Importantly, just as the future performance, $\rho(\pi_c, k+\delta)$, of $\pi_c$ is not known \emph{a priori} for a non-stationary MDP, the future performance of the baseline policy $\pi^\text{safe}$ is also not known \textit{a priori}.
%
Therefore, to ensure that the constraint in \eqref{eqn:constraint} is satisfied, we use $\mathscr C$ to obtain a \textit{one-sided} lower and upper bound for $\rho(\pi_c, {k+\delta})$ and $ \rho(\pi^\text{safe}, {k+\delta})$, respectively, each with confidence level $\alpha/2$.
The confidence level is set to  $\alpha/2$ so that the total failure rate (i.e., either $\rho(\pi_c, k+\delta)$ or $\rho(\pi^\text{safe}, k+\delta)$ is outside their respective bounds) is no more than $\alpha$.
Subsequently, $\texttt{alg}$ only updates $\pi^\text{safe}$ if the lower bound of $\rho(\pi_c, {k+\delta})$ is higher than the upper bound of $ \rho(\pi^\text{safe}, {k+\delta})$; otherwise, no update is made and $\pi^\text{safe}$ is chosen to be executed again.
%



%

\textbf{Candidate Policy Search: }
%
%
%
An ideal candidate policy $\pi_c$ would be one that has high future performance  $ \rho(\pi_c, {k+\delta})$, along with a large  confidence lower bound on its performance, so that it can pass the safety test.
However, in practice, there could often be conflicts between policies that might have higher estimated future performance but with lower confidence, and policies with lower  estimates of future performance but with higher confidence.
As the primary objective of our method is to ensure safety, we draw inspiration from prior methods for conservative/safe learning in stationary domains \cite{garcia2015comprehensive, thomas2015high, kazerouni2017conservative, chow2018lyapunov} and propose searching for a policy that has the \textit{highest lower} confidence bound.
That is, let the one-sided CI for the future  performance $\rho(\pi, k+\delta)$ obtained using $\mathscr C$ be $[\hat \rho^\text{lb}(\pi), \infty)$, then $ \pi_c \in \argmax_\pi \hat \rho^\text{lb}(\pi)$.

\textbf{Data-Splitting: } 
%
%
Conventionally, in the time-series literature, there is only a single trend that needs to be analyzed.
%
In our problem setup, however, the time series forecasting function is used to analyze trends of multiple policies during the candidate policy search.
%
%
If all of the available data $\mathcal D$ is used to estimate the lower bound $\hat \rho^\text{lb}(\pi)$ for $\rho(\pi, k+\delta)$ and if $\pi$ is chosen by maximizing $\hat \rho^\text{lb}(\pi)$, then due to the \textit{multiple comparisons problem} \citep{benjamini1995controlling}   we are likely to find a $\pi$ that over-fits to the data and achieves a higher value of $\hat \rho^\text{lb}(\pi)$.
A safety test based on such a $\hat \rho^\text{lb}(\pi)$ would thus be unreliable.
%
%
To address this problem, we partition $\mathcal D$ into two mutually exclusive sets, namely $\mathcal D_\text{train}$ and $\mathcal D_\text{test}$, such that only $\mathcal D_\text{train}$ is used to search for a candidate  policy $\pi_c$ and only $\mathcal D_\text{test}$ is used during the safety test.
%
%
%

\section{Estimating Confidence Intervals for Future Performance}
\label{sec:CI}
To complete the SPIN framework discussed in Section \ref{sec:spin}, we need to obtain an estimate $\hat \rho(\pi, k+\delta)$ of $\rho(\pi, k+\delta)$ and its confidence interval using the function $\mathscr C$.
This requires answering two questions:  
    (1) Given that in the past, policies $\{\beta_i\}_{i=1}^k$ were used to generate the observed returns, how do we estimate $\hat \rho(\pi, k+\delta)$ for a \textit{different} policy $\pi$? 
    (2)  Given that the trajectories are obtained only from a \textit{single} sample of the sequence  $\{M_i\}_{i=1}^{k}$, how do we obtain a confidence interval around $\hat \rho(\pi, k+\delta)$?  
We answer these two questions in this section. 

\textbf{Point Estimate of Future Performance:} 
To answer the first question, we build upon the following observation used by \citet{chandak2020future}:
While in the past, returns were observed by executing policies $\{\beta_i\}_{i=1}^k$, \textit{what if} policy $\pi$ was executed instead?  
%

Formally, we use per-decision importance sampling \citep{precup2000eligibility} for $H_i$, to obtain a \textit{counterfactual} estimate
%
$
     \hat \rho(\pi, i) \coloneqq  \sum_{t=0}^{\infty} \left ( \prod_{l=0}^t 
    \frac{\pi(A_i^l|S_i^l)}{\beta_i(A_i^l|S_i^l)} \right) \gamma^t R_i^t 
    $,
     of  $\pi$'s performance in the past episodes $i \in [1, k]$.
%
%
This estimate $\hat \rho(\pi, i)$ is an unbiased estimator of $\rho(\pi, i)$, i.e., $\mathbb{E}[\hat \rho(\pi,i)] = \rho(\pi, i)$, under the the following assumption \cite{thomas2015safe}, which can typically be satisfied using an entropy-regularized policy $\beta_i$.
\begin{ass}[Full Support]
\thlabel{ass:fullsupport}
    $\forall a \in \mathcal A$ and $\forall s \in \mathcal S$ there exists a $c > 0$ such that $\forall i, \beta_i(a|s) > c$.
\end{ass}
Having obtained counterfactual estimates $\{\hat \rho(\pi, i)\}_{i=1}^k$, we can then estimate $\rho(\pi, k+\delta)$ by analysing the performance trend of $\{\hat \rho(\pi, i)\}_{i=1}^k$ and forecasting the future performance $\hat \rho (\pi, {k+\delta})$.
That is, let $X \coloneqq [1, 2, ..., k]^\top  \in \mathbb{R}^{k \times 1}$, let $\Phi \in \mathbb{R}^{k \times d}$ be the corresponding basis matrix for $X$ such that $i^\text{th}$ row of $\Phi$, $\forall i \in [1,k]$, is  $\Phi_i \coloneqq \phi(X_i)$, and let
$   Y \coloneqq [\hat \rho(\pi, 1), \hat \rho(\pi, 2), ..., \hat \rho(\pi, k)]^\top \in \mathbb{R}^{k \times 1} $.
    Then under \thref{ass:smooth,ass:fullsupport}, an estimate $\hat \rho(\pi, k+\delta)$ of the future performance can be computed using least-squares (LS) regression, i.e.,
%
   $\hat \rho (\pi, {k+\delta}) = \phi(k+\delta)\hat w = \phi(k+\delta)(\Phi^\top  \Phi)^{-1}\Phi^\top Y.$
%
%

%
\myparagraph{Confidence Intervals for Future Performance: }
We now aim to quantify the uncertainty of $\hat \rho(\pi, k+\delta)$ using a confidence interval (CI), such that the true future performance $\rho(\pi, k+\delta)$ will be contained within the CI with the desired confidence level.
To obtain a CI for $\rho(\pi, k+\delta)$, we make use of $\texttt{t}$-statistics \cite{wasserman2013all} and use the following notation.
Let the sample standard deviation for $\hat \rho(\pi, k+\delta)$ be $\hat s$, where $\hat s^2 \coloneqq \phi(k+\delta)(\Phi^\top\Phi)^{-1}\Phi^\top  \hat \Omega \Phi (\Phi^\top\Phi)^{-1} \phi(k+\delta)^\top$, where $\hat \Omega$ is a diagonal matrix containing the square  of the regression errors $\hat \xi$ (see Appendix \ref{apx:sec:prelim} for more details), and
let the $\texttt{t}$-statistic be $\texttt{t} \coloneqq (\hat \rho(\pi, k+\delta) - \rho(\pi, k+\delta))/\hat s$. 

If the distribution of $\texttt{t}$ was known, then a $(1-\alpha)100\%$ CI could be obtained as $[\hat \rho (\pi, k+\delta) - \hat s \texttt{t}_{1 - \alpha/2}, \,\, \hat \rho (\pi, k+\delta) - \hat s \texttt{t}_{\alpha/2}]$, where for any $\alpha \in [0,1], \texttt{t}_\alpha$ represents $\alpha$-\textit{percentile} of the $\texttt{t}$ distribution. 
Unfortunately, the distribution of $\texttt{t}$ is not known.
One alternative could be to assume that $\texttt{t}$ follows the \textit{student}-\texttt{t} distribution \cite{student1908probable}.
However, that would only be valid if  \textit{all} the error terms in regression are \textit{homoscedastic} and \textit{normally} distributed.
Such an assumption could be severely violated in our setting due to the heteroscedastic nature of the estimates of the past performances resulting from the use of potentially different behavior policies $\{\beta_i\}_{i=1}^k$ and due to the unknown form of stochasticity in $\{M_i\}_{i=1}^k$.
Further, due to the use of importance sampling, the performance estimates $\{\hat \rho(\pi, i)\}_{i=1}^k$ can often be skewed and have a heavy-tailed distribution with high-variance \cite{thomas2015higheval}.
We provide more discussion on these issues in Appendix \ref{apx:sec:ttest}.

To resolve the above challenges, we make use of \textit{wild bootstrap}, a semi-parametric bootstrap procedure that is popular in time series analysis and econometrics \citep{wu1986jackknife, liu1988bootstrap, mammen1993bootstrap, davidson1999wild, davidson2008wild}. The idea is to generate multiple pseudo-samples of performance for each $M_i$, using the single sampled performance estimate. These multiple pseudo-samples can then be used to obtain an empirical distribution and so characterize the range of possible performances. This empirical distribution is for a pseudo $\texttt{t}$-statistic, $\texttt{t}^*$, where the $\alpha$-percentile for $\texttt{t}^*$ can be used to estimate the $\alpha$-percentile of the distribution of $\texttt{t}$. Below, we discuss how to get these multiple pseudo-samples.

Recall that trajectories $\{H_i\}_{i=1}^k$ are obtained only from a \textit{single} sample of the sequence $\{M_i\}_{i=1}^{k}$.
Due to this, only a single point estimate $\hat \rho(\pi, k+\delta)$, devoid of any estimate of uncertainty, of the future performance $\rho(\pi, k+\delta)$ can be obtained. Therefore, we aim to create \textit{pseudo-samples} of $\{\hat \rho(\pi, i)\}_{i=1}^k$ that \textit{resemble} the estimates of past performances that would have been obtained using trajectories from an alternate sample of the sequence $\{M_i\}_{i=1}^{k}$. The wild bootstrap procedure provides just such an approach, with the following steps.
\begin{enumerate}[leftmargin=*,itemsep=-1ex,topsep=0pt]
    \item Let  $Y^+ \coloneqq [\rho(\pi, 1), ..., \ \rho(\pi, k)]^\top \in \mathbb{R}^{k \times 1} $ correspond to the true performances of $\pi$. 
    Create $\hat Y = \Phi(\Phi^\top  \Phi)^{-1}\Phi^\top Y$, an LS estimate of $Y^+$, using the counterfactual performance estimates $Y$
    and obtain the regression errors  $\hat \xi \coloneqq \hat Y - Y$.
    \item Create pseudo-noises $\xi^* \coloneqq \hat \xi \odot \sigma^*$, where $\odot$ represents Hadamard product and $\sigma^* \in \mathbb{R}^{k \times 1}$ is Rademacher random variable (i.e., $\forall i \in [1,k], \,\, \Pr(\sigma_i^*=+1) = \Pr(\sigma_i^*=-1) = 0.5$).\footnote{While in machine learning the `*' symbol is often used to denote optimal variables, to be consistent with the bootstrap literature our usage of this symbol denotes pseudo-variables.}
    \item Create pseudo-performances  $Y^* \coloneqq \hat Y + \xi^*$, to obtain pseudo-samples for $\hat Y$ and $\hat \rho (\pi, {k+\delta})$ as $\hat Y^* = \Phi(\Phi^\top  \Phi)^{-1}\Phi^\top Y^*$ and  $\hat \rho (\pi, {k+\delta})^* = \phi(k+\delta)(\Phi^\top  \Phi)^{-1}\Phi^\top Y^*.$
\end{enumerate}
Steps 2 and 3 can be repeated to re-sample up to $B \leq 2^k$ \textit{similar} sequences of past performance $Y^*$, from a \textit{single} observed sequence $Y$ of length $k$, while also preserving the time-series structure. This unreasonable property led \citet{mammen1993bootstrap} to coin the term `wild bootstrap'.
For a brief discussion on \textit{why} wild bootstrap works, see Appendix \ref{apx:sec:whyboot}.

Given these multiple pseudo-samples, we can now obtain an empirical distribution for pseudo $\texttt{t}$-statistic, $\texttt{t}^*$.
Let the pseudo-sample standard deviation be $\hat s^*$, where $\hat s^{*2} \coloneqq \phi(k+\delta)(\Phi^\top\Phi)^{-1}\Phi^\top  \hat \Omega^* \Phi (\Phi^\top\Phi)^{-1} \phi(k+\delta)^\top$, where $\hat \Omega^*$ is a diagonal matrix containing the square of the pseudo-errors $\hat \xi^* \coloneqq \hat Y^* -  Y^*$.
Let $\texttt{t}^* \coloneqq (\hat \rho(\pi, k+\delta)^* - \hat \rho(\pi, k+\delta))/\hat s^*$.
Then an $\alpha$-percentile $\texttt{t}^*_\alpha$ of the empirical distribution of $\texttt{t}^*$ is used to estimate the $\alpha$-percentile of $\texttt{t}$'s distribution.

Finally, we can define $\mathscr C$ to use the wild bootstrap to produce CIs.
To ensure this is principled, we leverage a property proven by \citet{djogbenou2019asymptotic} and show in the following theorem that the CI for $\rho(\pi, k+\delta)$ obtained using pseudo-samples from wild bootstrap is \textit{consistent}.
For simplicity, we restrict our focus to settings where $\phi$ is the Fourier basis (see Appendix \ref{apx:sec:coverageproof} for more discussion).
\begin{thm}[Consistent Coverage] 
\thlabel{thm:consistentcoverage}
Under \thref{ass:smooth,ass:fullsupport}, if the trajectories $\{H_i\}_{i=1}^k$ are independent and if $\phi(x)$ is a Fourier basis, then as $k \rightarrow \infty$,
\begin{align}
\Pr\left (\rho(\pi, k+\delta) \in  \left [\hat \rho (\pi, k+\delta) - \hat s \texttt{t}^*_{1 - \alpha/2}, \,\, \hat \rho (\pi, k+\delta) - \hat s \texttt{t}^*_{\alpha/2} \right] \right) \rightarrow 1 - \alpha. 
\end{align}
\end{thm}

\textbf{Remark:} We considered several factors when choosing the wild bootstrap to create pseudo-samples of $\hat \rho(\pi, {k+\delta})$:
(a) because of the time-series structure, there exists no joint distribution between the deterministic sequence of time indices, $X$, and the stochastic performance estimates, $Y$, (b) trajectories from only a \textit{single} sequence of $\{M_i\}_{i=1}^k$ are observed, (c) trajectories could have been generated using different $\beta_i$'s leading to \textit{heteroscedasticity} in the performance estimates $\{\hat \rho(\pi, i)\}_{i=1}^k$, (d) different policies $\pi$ can lead to different distributions of performance estimates, even for the same behavior policy $\beta$, and (e) even for a fixed $\pi$ and $\beta$, performance estimates $\{\hat \rho(\pi, i)\}_{i=1}^k$ can exhibit heteroskedasticty due to inherent stochasticity in $\{M_i\}_{i=1}^k$ as mentioned in \thref{ass:smooth}.
These factors make popular approaches like pairs bootstrap, residual bootstrap, and block bootstrap not suitable for our purpose. 
In contrast, the wild bootstrap can take all these factors into account. 
More discussion on other approaches is available in Appendix \ref{apx:sec:otherboot}.

\section{Implementation Details}

\begin{figure}
    \centering
    \includegraphics[width=0.95\textwidth]{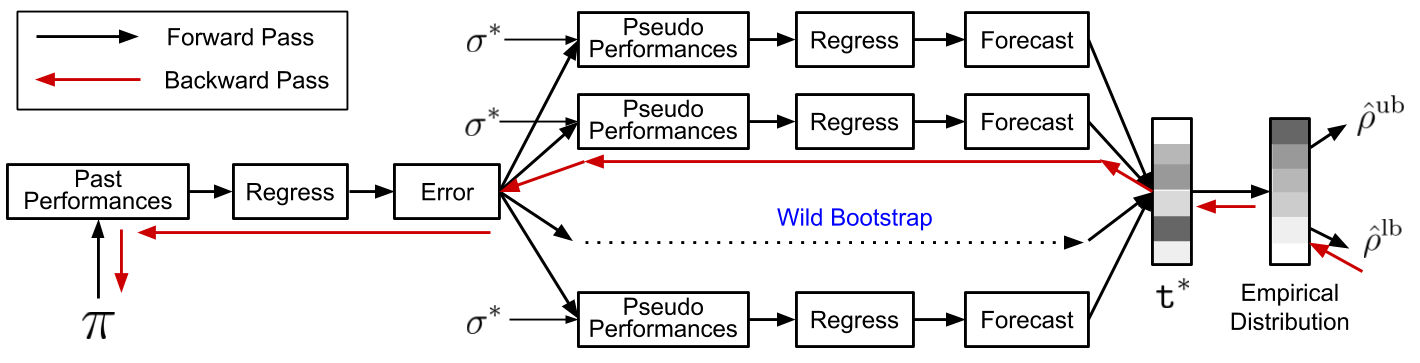}
    \caption{To search for a candidate policy $\pi_c$, regression is first used to analyze the trend of a given policy's past performances. 
    Wild bootstrap then leverages Rademacher variables $\sigma^*$ and the errors from regression to create pseudo-performances.
    Based on these pseudo-performances, an empirical distribution of the pseudo $\texttt{t}$-statistic, $\texttt{t}^*$, of the estimate of future performance, is obtained.
    The candidate policy $\pi_c$ is found using a differentiation based optimization procedure that maximizes the lower bound, $\hat \rho^\text{lb}$, computed using the empirical distribution of $\texttt{t}^*$.
    \vspace{-5pt}}
    \label{fig:diffLB}
\end{figure}

Notice that as the CI $[\hat \rho^\text{lb}(\pi), \hat \rho^\text{ub}(\pi)]$ obtained from $\mathscr C$ is based on the wild bootstrap procedure, a gradient based optimization procedure for maximizing the lower bound $\hat \rho^\text{lb}(\pi)$ would require differentiating through the entire bootstrap process.
Figure \ref{fig:diffLB} illustrates the high-level steps in this optimization process.
More elaborate details and complete algorithms are deferred to Appendix \ref{sec:apx:algorithm}.

Further, notice that a smaller amount of data results in greater uncertainty and thus wider CIs.
%
While a tighter CI during candidate policy search can be obtained by combining all the past $\mathcal D_\text{train}$ to increase the amount of data,  each safety test should ideally be independent of all the previous tests, and should therefore use data that has never been used before.
While it is possible to do so, using only new data for each safety test would be data-inefficient.
%

To make our algorithm more data efficient, similar to the approach of \citet{thomas2019preventing}, 
we re-use the test data in subsequent tests.
As illustrated by the black dashed arrows in Figure \ref{fig:Seldonian}, this modification introduces a subtle source of error because the data used in consecutive tests are not completely independent. 
However, the practical advantage of this approach in terms of tighter confidence intervals can be significant.
Further, as we demonstrate empirically, the error introduced by re-using test data can be negligible in comparison to the error due to the false assumption of stationarity.  

\section{Empirical Analysis}

In this section, we provide an empirical analysis on two domains inspired by safety-critical real-world problems that exhibit non-stationarity.
In the following, we first briefly discuss these domains, and in Figure \ref{fig:results} we present a summary of results for eight settings (four for each domain).
A more detailed description of the domains and the experimental setup is available in Appendix \ref{apx:sec:experiments}.

\noindent\textbf{Non-Stationary Recommender System (RecoSys):} 
In this domain, a synthetic recommender system interacts with a user whose interests in different products change over time.
Specifically, the reward for recommending each product varies in a seasonal cycle.
Such a scenario is ubiquitous in industrial applications, and updates to an existing system should be made responsibly; if it is not ensured that the new system is better than the existing one, then it might result in a loss of revenue.

\begin{figure}
    \centering
    \includegraphics[width=0.9\textwidth]{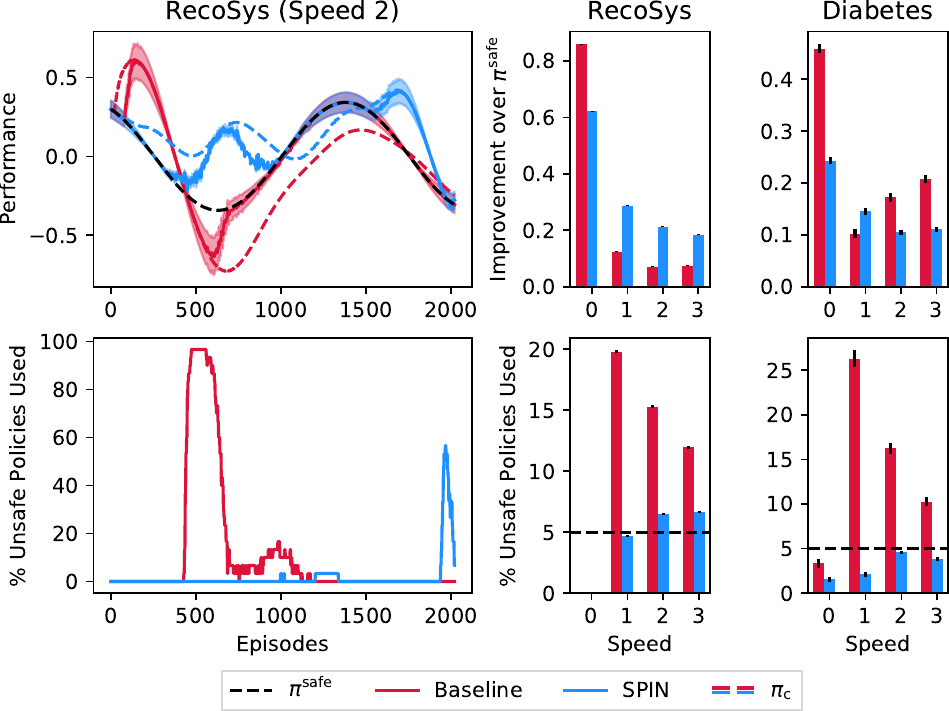}
    \caption{
    (Top-left) An illustration of a typical learning curve. Notice that SPIN updates a policy whenever there is room for a significant improvement.
    %
    %
    %
    (Middle and Right) 
    As our main goal is to ensure safety, \textit{while being robust to how a user of our algorithm sets the hyper-parameters (HPs)}, we do \textit{not} show results from the best HP.
    This choice is motivated by the fact that best performances can often be misleading as it only shows what an algorithm \textit{can} achieve and not what it is \textit{likely} to achieve \cite{jordan2018using}.
    Therefore, we present the aggregated results averaged over the \textit{entire sweep} of $1000$ HPs per algorithm, per speed, per domain. 
   Shaded regions and intervals correspond to the standard error.
    See Appendix \ref{apx:sec:hyper} and \ref{apx:sec:plots} for more details on these plots and on the data aggregation process.
    \vspace{-15pt}
    }
    \label{fig:results}
\end{figure}

\noindent\textbf{Non-Stationary Diabetes Treatment: }
This environment is based on an open-source implementation \cite{simglucose} of the FDA approved Type-$1$ Diabetes Mellitus simulator (T1DMS) \cite{man2014uva,kovatchev2009silico} for treatment of Type-1 diabetes.
Each episode consists of a day ($1440$ timesteps--one for each minute in a day) in an \textit{in-silico} patient’s life and the blood-glucose level increases when a patient consumes a meal.
Too high or too low blood-glucose can lead to fatal conditions like hyperglycemia and hypoglycemia, respectively, and an insulin dosage is therefore required to minimize the risk associated with these conditions.
While a doctor's initial dosage prescription is often available, the insulin sensitivity of a patient’s internal body organs varies over time, inducing non-stationarity that should be accounted for.
%
%
The goal of the system is to responsibly update the doctor's initial prescription, ensuring that the treatment is only made better. 

The diabetes treatment problem is particularly challenging as the performance trend of policies in this domain can violate \thref{ass:smooth} (see Appendix \ref{apx:sec:domains} for more discussion).
This presents a more realistic challenge as every policy's performance trend for real-world problems cannot be expected to follow \textit{any} specific trend \textit{exactly}--one can only hope to obtain a coarse approximation of the trend.

For both the domains, (a) we set $\pi^\text{safe}$ to a near-optimal policy for the starting MDP $M_1$, representing how a doctor would have set the treatment initially, or how an expert would have set the recommendations,
(b) we set the safety level $(1-\alpha)$ to $95\%$,
(c) we modulate the ``speed'' of non-stationarity, such that higher speeds represent a faster rate of non-stationarity and a speed of zero represents a stationary domain, and
(d) we consider the following two algorithms for comparison: (i) \textbf{SPIN: } The proposed algorithm that ensures safety while taking into account the impact of non-stationarity, and (ii) \textbf{Baseline: } An algorithm similar to those used in prior works \cite{thomas2015high,thomas2019preventing,metevier2019offline}, which is aimed at ensuring safety but ignores the impact of non-stationarity (see Appendix \ref{apx:sec:baseline} for details).

\myparagraph{Results: }
An ideal algorithm should adhere to the safety constraint in \eqref{eqn:constraint}, maximize future performance, and also be robust to hyper-parameters even in the presence of non-stationarity.
Therefore, to analyse an algorithm's behavior, we aim to investigate the following three questions:

\textbf{Q1: }  \textit{How often does an algorithm violate the safety constraint $\rho(\texttt{alg}(\mathcal D), {k+\delta}) \geq \rho(\pi^\text{safe}, k + \delta)$?}
We present these results for SPIN and Baseline on both the domains in Figure \ref{fig:results} (bottom).
Baseline ensures safety for the stationary setting (speed $=0$) but has a severe failure rate otherwise.
Perhaps counter-intuitively, the failure rate for Baseline is much \textit{higher} than $5\%$ for \textit{slower} speeds.
This can be attributed to the fact that at higher speeds, greater reward fluctuations result in more variance in the performance estimates, causing the CIs within Baseline to be looser, and thereby causing Baseline to have insufficient confidence of policy improvement to make a policy update. 
Thus, at higher speeds Baseline becomes safer as it reverts to $\pi^\text{safe}$ more often.
This calls into question 
the popular misconception that the stationarity assumption is not severe when changes are slow, as in practice slower changes might be harder for an algorithm to identify, and thus might jeopardize safety. 
By comparison, even though bootstrap CIs do not have guaranteed coverage when using a finite number of samples \cite{efron1994introduction}, it still allows SPIN to maintain a failure rate near the $5\%$ target.

\textbf{Q2: }  \textit{What is the performance gain of an algorithm over the existing known safe policy $\pi^\text{safe}$?}
Notice that any algorithm $\texttt{alg}$ can satisfy the safety constraint in \eqref{eqn:constraint} by \textit{never} updating the existing policy $\pi^\text{safe}$.
Such an \texttt{alg} is not ideal as it will provide no performance gain over $\pi^\text{safe}$.
In the stationary settings, Baseline provides better performance gain than SPIN while maintaining the desired failure rate.
However, in the non-stationary setting, the performance gain of SPIN is higher for the recommender system. 
For diabetes treatment,  both the methods provide similar performance gain but only SPIN does so while being safe (see the bottom-right of Figure \ref{fig:results}).
%
The similar performance of Baseline to SPIN despite being unsafe can be attributed to occasionally deploying better policies than SPIN, but having this improvement negated by deploying policies worse than the safety policy (e.g., see the top-left of Figure \ref{fig:results}).

\textbf{Q3:} \textit{How robust is SPIN to hyper-parameter choices?}
To analyze the robustness of our method to the choice of  relative train-test data set sizes, the objective for the candidate policy search, and to quantify the benefits of the proposed safety test, we provide an ablation study on the RecoSys domain, for all speeds ($0,1,2,3$) in Table \ref{tab:1}.
All other experimental details are the same as in Appendix E.3, except for (iv), where mean performance, as opposed to the lower bound, is optimized during the candidate search. 
Table \ref{tab:1} shows that the safety violation rate of SPIN is robust to such hyper-parameter changes.
However, it is worth noting that too small a test set can make it harder to pass the safety-test, 
and so performance improvement is small in (i).
In contrast, if the proposed safety check procedure for a policy's performance on a non-stationary MDP is removed, then the results can be catastrophic, as can be seen in (v).

\begin{table}[t]
    \centering
    \begin{tabular}{ccc|cccc|cccc}
        & & train-test & 0 & 1 & 2 & 3 & 0 & 1& 2& 3 \\
        \hline
        (i) & SPIN & 75\%--25\%  & .56 & .22 & .17 & .14 & 0.0 & 3.6 & 5.1 & 5.4 \\
        (ii)  & SPIN & 25\%--75\% & .48 & .29 & .21 & .19 & 0.0 & 4.6 & 6.5 & 7.0 \\
        (iii) & SPIN  & 50\%--50\% & .62 & .28 & .21 & .18 & 0.0 & 4.7 & 6.4 & 6.6\\
        (iv) & SPIN-mean & 50\%--50\% & .70 & .28 & .24 & .19 & 0.2 & 4.9 & 6.3 & 7.1 \\
        \hline
        (v) & NS + No safety & 100\%--0\% & .73 & .22 & .16  &.19 & 9.4 & 37.6 & 40.2 & 38.6 \\
        (vi)& Stationary + Safety & 50\%--50\% & .85 & .12 & .07 & .07 & 0.0 & 19.8 & 15.3 & 11.9 
    \end{tabular}
    \caption{Ablation study on the RecoSys domain. (Left) Algorithm. (Middle) Improvement over $\pi^\text{safe}$. (Right) Safety violation percentage. Rows (iii) and (vi) correspond to results in Figure \ref{fig:results}.}
    \label{tab:1}
    \vspace{-20pt}
\end{table}

\section{Conclusion}

In this paper, we took several first steps towards  ensuring safe policy improvement for NS-MDPs.
We discussed the difficulty of this problem and presented an algorithm for ensuring the safety constraint in \eqref{eqn:constraint} under the assumption of a smooth performance trend.  
Further, our experimental results call into question 
the popular misconception that the stationarity assumption is not severe when changes are slow. 
In fact, it can be quite the opposite: Slow changes can be more \textit{deceptive} and can make existing algorithms,  which do not account for non-stationarity, more susceptible to deploying unsafe policies.
%
%
%

\section{Acknowledgement}
We are thankful to Prof. James MacKinnon for sharing valuable insights and other useful references for the wild bootstrap technique.
We are also thankful to Shiv Shankar and the anonymous reviewers for providing feedback that helped improve the paper.

This work was supported in part by NSF Award \#2018372 and gifts from Adobe Research.
Further, this work was also supported in part by NSERC and CIFAR, particularly through funding the Alberta Machine Intelligence Institute (Amii) and  the CCAI Chair  program.

Research reported in this paper was also sponsored in part by the CCDC Army Research Laboratory under Cooperative Agreement W911NF-17-2-0196 (ARL IoBT CRA). The views and conclusions contained in this document are those of the authors and should not be interpreted as representing the official policies, either expressed or implied, of the Army Research Laboratory or the U.S. Government. The U.S. Government is authorized to reproduce and distribute reprints for Government purposes notwithstanding any copyright notation herein.

\section{Broader Impact} 
\label{sec:broaderImpactpst}
\paragraph{Applications:} 
We hope that our work brings more attention to the understudied challenge of ensuring safety that is critical for the responsible application of RL algorithms to real-world non-stationary problems.
%
%
For example, researchers have proposed the use of reinforcement learning algorithms for several medical support systems, ranging from diabetes management \citep{bastani2014model}, to epilepsy \citep{pineau2009treating}, to sepsis treatment \citep{saria2018individualized}.
These problems involve sequential decision-making, where autonomous systems can improve upon a doctor's prescribed policy by adapting to the non-stationary dynamics of the human body as more data becomes available.
In fact, almost all human-computer interaction systems (medical treatment, tutorial recommendations, advertisement marketing, etc.) have a common non-stationary component: humans.
Also, in all these use-cases, it is important to ensure that the updates are safe.
That is, the updated system should not lead to undesirable financial/medical conditions and should only improve upon the existing policy (e.g., doctor's initial prescription).

\paragraph{Ethical concerns: } The proposed method is focused towards ensuring \textit{safety}, defined in terms of the performance of a system.
The proposed algorithm to do so makes use of data generated by interacting with a non-stationary MDP.
As discussed above, in many cases, non-stationary MDPs are associated with human beings.
This raises additional issue of \textit{safety} concerning data privacy and security.
The proposed method \textit{does not} resolve any of these issues, and therefore additional care should be taken for adequate data management.

\paragraph{Note to a wider audience:}
The proposed method relies upon  smoothness assumptions that need not be applicable to all problems of interests.
For example, when there are jumps or breaks in the time series, then the behavior of the proposed method is not ensured to be safe. 
Our method also makes use of importance sampling which requires access to the probabilities of the past actions taken under the behavior policy $\beta$. 
If these probabilities are not available and are instead estimated from data then it may introduce bias and may result in a greater violation of the safety constraint.
Further, all of our experiments were conducted on simulated domains, where the exact nature of non-stationarity may \textit{not} reflect the non-stationarity observed during actual interactions in the physical world.
Developing simulators that closely mimic the physical world, without incorporating systematic and racial bias, remains an open problem and is complementary to our research.
Hence, caution is warranted while quoting results from these simulation experiments.

\paragraph{Future research directions: } There are several exciting directions for future research.
We used the ordinary importance sampling procedure to estimate past performances of a policy. 
However, it suffers from high variance and leveraging better importance sampling procedures \cite{jiang2015doubly,thomas2016data} can be directly beneficial to obtain better estimates of past performances.
Leveraging time-series models like ARIMA \cite{chen2009arima} and their associated wild-bootstrap methods \cite{godfrey2005wild, djogbenou2015bootstrap, friedrich2020autoregressive} can be a fruitful direction for extending our algorithm to more general settings that have correlated noises or where the performance trend, both locally and globally, can be better modeled using auto-regressive functions.
Further, goodness-of-fit tests \cite{chen2003empirical} could be used to search for a time-series model that best fits the application.

\bibliography{mybib}
\bibliographystyle{abbrvnat}

\clearpage
\appendix

\setcounter{thm}{0}

\onecolumn

\begin{center}
    \Large
    \textbf{Towards Safe Policy Improvement 
    for \\ Non-Stationary MDPs (Supplementary Material)}
\end{center}

\section{Notation}

\begin{table}[h]
    \centering
    \begin{tabular}{c|l c}
    \hline \\
    Symbol & Meaning \\
    \hline \\
    $M_i$  & MDP for episode $i$. \\
     $\mathcal S$    & State set.  \\
    $\mathcal A$     & Action set. \\
    $\mathcal P_i$ & Transition dynamics for $M_i$. \\
    $\mathcal R_i$ & Reward function for $M_i$. \\
    $\gamma$ & Discounting factor. \\
    $d_0$ & Starting state distribution. \\
    $\pi$ & Policy. \\
    $\pi^\text{safe}$ & Given baseline safe policy. \\
    $\pi_c$ & A candidate policy that can possibly be used for policy improvement. \\
    $\beta_i$ & Behavior policy used to collect data for episode $i$. \\
    $G(\pi, m)$ & Discounted episodic return of $\pi$ for MDP $m$. \\
    $\rho(\pi, m)$ & Expected discounted episodic return of $\pi$ for MDP $m$. \\
    $\rho(\pi, i)$ & Expected discounted episodic return for episode $i$. \\
    $\hat \rho(\pi, i)$ & An estimate of $\rho(\pi, i)$. \\
    $\hat \rho^\text{lb}(\pi)$ & Lower bound on the future performance of $\pi$. \\
    $\hat \rho^\text{ub}(\pi)$ & Upper bound on the future performance of $\pi$. \\
    $k$ & Current episode number. \\
    $\delta$ & Number of episodes into the future. \\
    $H_i$ & Trajectory during episode $i$. \\
    $\mathcal D$ & Set of trajectories. \\
    $\mathcal D_\text{train}$ & Partition of $\mathcal D$ used for searching $\pi_c$. \\
    $\mathcal D_\text{test}$ & Partition of $\mathcal D$ used for safety test. \\
    $\texttt{alg}$ & An algorithm. \\
    $\alpha$ & Quantity to define the desired safety level $1 - \alpha$. \\
    $X$ & Time indices for time-series. \\
    $Y$ & Time series values corresponding to $X$. \\
    $\hat Y$ & Estimates for $Y$. \\
    $\phi$ & Basis function for time series forecasting. \\
    $\Phi$ & Matrix containing basis for different episode numbers. \\
    $w$ & parameters for time series forecasting. \\
    $\xi$ & Noise in the observed performances. \\
    $\hat \xi$ & Estimate for $\xi$. \\ 
    $\hat \Omega$ & Diagonal matrix containing $\hat \xi^2$. \\
    $\hat s$ & Standard deviation of the forecast. \\
    $\texttt{t}$ & $\texttt{t}-$statistic for the forecast. \\
    $\texttt{t}_\alpha$ & $\alpha-$quantile of the $\texttt{t}$ distribution. \\
    $\mathscr C$ & Function to obtain confidence interval on future performance (using wild bootstrap). \\
    $\sigma^*$ & Rademacher random variable. \\
    $Y^*$ & Pseudo-variable for $Y$.\\
    $\hat Y^*$ & Pseudo-variable for $\hat Y$.\\
    $\xi^*$ & Pseudo-variable for $\xi$.\\
    $\hat \xi^*$ & Pseudo-variable for $\hat \xi$.\\
    $\texttt{t}^*$ & Pseudo-variable for $\texttt{t}$.\\
    $\texttt{t}_\alpha^*$ & Pseudo-variable for $ \texttt{t}_\alpha$.\\
    $\hat s^*$ & Pseudo-variable for $\hat s$. 
    \end{tabular}
    \caption{List of symbols used in the main paper and their associated meanings.}
    \label{tab:my_label}
\end{table}

\clearpage

\section{Hardness Results}
\label{apx:proof}
Several works in the past have presented performance bounds for a policy when executed on an approximated stationary MDP \cite{whitt1978approximations,kakade2002approximately,kearns2002near,ravindran2004approximate,pirotta2013safe,achiam2017constrained}.
See Section 6 by \citet{neuro} for a textbook reference. 
The technique of our proof for \thref{thm:lipbound} regarding non-stationary MDPs is based on these earlier results.

\begin{thm}[Lipschitz smooth performance]
If $\exists \epsilon_P \in \mathbb{R}$ and $\exists \epsilon_R \in \mathbb{R}$ such that for any $M_k$ and $M_{k+1}$,
$
\forall s \in \mathcal S, \forall a \in \mathcal A,\,\,\, \lVert \mathcal P_{k}(\cdot| s, a) -\mathcal  P_{k+1}(\cdot| s, a)\rVert_1 \leq \epsilon_P \,\, \text{and} \,\, | \mathbb{E}[\mathcal R_{k}(s, a)] - \mathbb{E}[\mathcal  R_{k+1}(s, a)]| \leq \epsilon_R
$,
then the performance of any policy $\pi$ is Lipschitz smooth over time, with Lipschitz constant $L \coloneqq  \left (\frac{\gamma R_{\text{max}}}{(1 - \gamma)^2}\epsilon_P + \frac{1}{1 - \gamma}\epsilon_R \right)$.
That is,
\begin{align}
    \forall k \in \mathbb{N}_{> 0}, \forall \delta \in \mathbb{N}_{> 0}, \,\, \,\, |\rho(\pi, k) - \rho(\pi, {k+\delta})| \leq  L \delta. \label{eqn:apx:lipbound}
\end{align}
\end{thm}
\begin{proof}
We begin by noting that,
\begin{align}
    |\rho(\pi, k) - \rho(\pi, {k+\delta})| \leq \underset{M_k \in \mathcal{M}, M_{k+\delta} \in \mathcal{M}}{\text{sup}} |\rho(\pi, M_k) - \rho(\pi, M_{k+\delta})| \label{apx:eqn:supM}.
\end{align}
We now aim at bounding $|\rho(\pi, M_k) - \rho(\pi, M_{k+\delta})|$ in \eqref{apx:eqn:supM}.
Let $R_k(s,a) = \mathbb{E}[\mathcal R_k(s,a)]$, then notice that the on-policy distribution and the performance of a policy $\pi$ in the episode $k$ can be written as,
\begin{align}
    d^\pi(s, M_k) =& (1 - \gamma)\sum_{t=0}^{\infty} \gamma^t \Pr(S_t=s|\pi, M_k), \\
    \rho(\pi, M_k) =& (1-\gamma)^{-1}\sum_{s \in \mathcal S}d^\pi(s, M_k) \sum_{a \in \mathcal A} \pi(a|s) R_k(s,a).
    %
\end{align}
We begin the proof by expanding the absolute difference between the two performances as follows: 
\begin{align}
    &\lvert \rho(\pi, M_{k}) - \rho(\pi, M_{k+\delta}) \rvert \\
    =& \lvert \rho (\pi, M_{k}) - \rho(\pi, M_{k+1}) + \rho(\pi, M_{k+1}) - ... - \rho(\pi, M_{k+\delta-1}) + \rho(\pi, M_{k+\delta-1}) - \rho(\pi, M_{k+\delta}) \rvert \\
    \leq& \sum_{i=k}^{k+\delta-1}  \lvert \rho(\pi, M_{i}) - \rho(\pi, M_{i+1}) \rvert. \label{eqn:consecutive}
\end{align}
To simplify further, we introduce a temporary notation $\Delta(s,a) \coloneqq R_i(s,a) - R_{i+1}(s,a)$.
Now on expanding each of the consecutive differences in \eqref{eqn:consecutive} and multiplying by $(1-\gamma)$ on both sides:
\begin{align}
     & (1-\gamma)\lvert \rho(\pi, M_{i}) - \rho(\pi, M_{i+1}) \rvert 
        \\
      =& 
      \left \lvert  \sum_{s \in \mathcal S}d^\pi(s, M_i) \sum_{a \in \mathcal A} \pi(a|s) R_i(s,a) -  \sum_{s \in \mathcal S}d^\pi(s, M_{i+1}) \sum_{a \in \mathcal A} \pi(a|s) R_{i+1}(s,a)  \right \rvert \\
      =& 
      \left \lvert  \sum_{s \in \mathcal S} \sum_{a \in \mathcal A} \pi(a|s) \Big ( d^\pi(s, M_i)  R_i(s,a) - d^\pi(s, M_{i+1}) R_{i+1}(s,a)  \Big )   \right \rvert \\
      =& 
      \left \lvert  \sum_{s \in \mathcal S}\sum_{a \in \mathcal A} \pi(a|s)  \Big ( d^\pi(s, M_i)  (R_{i+1}(s,a) + \Delta(s,a)) - d^\pi(s, M_{i+1}) R_{i+1}(s,a)  \Big )   \right \rvert \\
      =& 
      \left \lvert  \sum_{s \in \mathcal S}\sum_{a \in \mathcal A} \pi(a|s)  \Big ( d^\pi(s, M_i) - d^\pi(s, M_{i+1})  \Big ) R_{i+1}(s,a) + \sum_{s \in \mathcal S}\sum_{a \in \mathcal A} \pi(a|s) d^\pi(s, M_i)  \Delta(s,a) \right \rvert. \label{apx:eqn:expanded}
\end{align}

In the following, we bound the terms in \eqref{apx:eqn:expanded} using the following three steps, 
(a) use Cauchy Schwartz inequality and bound each possible negative term with its absolute value,
(b) bound each reward $R_{i+1}(s,a)$ using $R_\text{max}$ and use the Lipschitz smoothness assumption to bound each $\Delta(s,a)$ using $\epsilon_R$, and
(c) equate sum of probabilities to one.
Formally, 

\begin{align}
     & (1 - \gamma) \lvert \rho(\pi, M_{i}) - \rho(\pi, M_{i+1}) \rvert 
     \\
      \overset{(a)}{\leq}&    \sum_{s \in \mathcal S} \sum_{a \in \mathcal A}  \pi(a|s) \left \lvert d^\pi(s, M_i) - d^\pi(s, M_{i+1})  \right \rvert \left \lvert R_{i+1}(s,a) \right \rvert +  \sum_{s \in \mathcal S}\sum_{a \in \mathcal A} \pi(a|s) d^\pi(s, M_i) \left \lvert \Delta(s,a) \right \rvert \\
      \overset{(b)}{\leq}&   R_{\text{max}}\sum_{s \in \mathcal S} \sum_{a \in \mathcal A} \pi(a|s) \left \lvert d^\pi(s, M_i) - d^\pi(s, M_{i+1})  \right \rvert +  \epsilon_R \sum_{s \in \mathcal S}\sum_{a \in \mathcal A} \pi(a|s) d^\pi(s, M_i)\\
      \overset{(c)}{=}&  R_{\text{max}}\sum_{s \in \mathcal S} \left \lvert d^\pi(s, M_i) - d^\pi(s, M_{i+1})  \right \rvert +  \epsilon_R. \label{apx:eqn:dbound}
\end{align}
To simplify \eqref{apx:eqn:dbound} further, we make use of the following property,

\begin{prop}[\citet{achiam2017constrained}]
    \label{prop:distributionbound}
    Let $P_i^\pi \in \mathbb{R}^{|\mathcal S| \times |\mathcal S|} $ be the transition matrix ($s'$ in rows and $s$ in columns) resulting due to $\pi$ and $P_i$, i.e., $\forall t, \,\, P_i^\pi(s',s) \coloneqq \Pr(S_{t+1}=s'|S_t=s, \pi, M_i)$, and let $d^{\pi}(\cdot, M_i) \in \mathbb{R}^{|\mathcal S|}$ denote the vector of probabilities for each state,  then\footnote{Note that the original result by \citet{achiam2017constrained} bounds the change in distribution between two different policies under the same dynamics. Here, we have modified the property for our case, where the policy is fixed but the dynamics are different.}
     \begin{align}
            \sum_{s \in \mathcal S}  \lvert d^\pi(s, M_i) - d^\pi(s, M_{i+1}) \rvert   \leq \gamma (1 - \gamma)^{-1} \left \lVert (P_i^{\pi} - P_{i+1}^{\pi}) d^{\pi}(\cdot, M_i)  \right \rVert _1.
    \end{align}
\end{prop}

Using Property \ref{prop:distributionbound},
     \begin{align}
          & \sum_{s \in \mathcal S}  \left \lvert  d^\pi(s,M_i) - d^\pi(s, M_{i+1}) \right \rvert \\
          \overset{(d)}{\leq}& \gamma (1 - \gamma)^{-1}  \sum_{s' \in \mathcal S} \left \lvert \sum_{s \in \mathcal S} \left ( P_i^\pi(s', s) -  P_{i+1}^\pi(s', s)\right ) d^{\pi}(s, M_i)  \right \rvert \\
          \leq& \gamma (1 - \gamma)^{-1}  \sum_{s' \in \mathcal S}  \sum_{s \in \mathcal S} \left \lvert  P_i^\pi(s',s) -  P_{i+1}^\pi(s',s) \right \rvert d^{\pi}(s, M_i)  \\
          =& \gamma (1 - \gamma)^{-1}  \sum_{s' \in \mathcal S}  \sum_{s \in \mathcal S} \left \lvert \sum_{a \in \mathcal A}  \pi(a|s) \Big (   \Pr(s'|s, a, M_i) -  \Pr(s'|s, a, M_{i+1}) \Big) \right \rvert d^{\pi}(s, M_i) \\  
          \leq& \gamma (1 - \gamma)^{-1}  \sum_{s' \in \mathcal S}  \sum_{s \in \mathcal S} \sum_{a \in \mathcal A}  \pi(a|s) \left \lvert  \Pr(s'|s, a, M_i) -  \Pr(s'|s, a, M_{i+1}) \right \rvert d^{\pi}(s, M_i)  
          \\
          =& \gamma (1 - \gamma)^{-1}    \sum_{s \in \mathcal S} \sum_{a \in \mathcal A}  \pi(a|s)  d^{\pi}(s, M_i) \sum_{s' \in \mathcal S} \left \lvert  \Pr(s'|s, a, M_i) -  \Pr(s'|s, a, M_{i+1}) \right \rvert  
          \\
          \overset{(e)}{\leq}& \gamma (1 - \gamma)^{-1}   \sum_{s \in \mathcal S} \sum_{a \in \mathcal A}  \pi(a|s)  d^{\pi}(s, M_i) \epsilon_P \\
          =& \gamma (1 - \gamma)^{-1} \epsilon_P, \label{eqn:distribution}
    \end{align}
where (d) follows from expanding the L1 norm of a matrix-vector product, and (e) follows from using the Lipschitz smoothness to bound the difference between successive transition matrices.
Combining \eqref{apx:eqn:dbound} and \eqref{eqn:distribution}, 

\begin{align}
    \lvert \rho(\pi, M_{i}) - \rho(\pi, M_{i+1}) \rvert \leq& (1 - \gamma)^{-1}  \left ( R_{\text{max}} \gamma (1 - \gamma)^{-1} \epsilon_P +  \epsilon_R  \right ). \\
    =& \frac{\gamma R_{\text{max}}}{(1 - \gamma)^2}\epsilon_P + \frac{1}{1 - \gamma}\epsilon_R. \label{eqn:single}
\end{align}
Finally, combining \eqref{eqn:consecutive} and \eqref{eqn:single}, 
\begin{align}
    \lvert \rho(\pi, M_{i}) - \rho(\pi, M_{i+\delta}) \rvert    \leq& \sum_{i=k}^{k+\delta-1} \left (\frac{\gamma R_{\text{max}}}{(1 - \gamma)^2}\epsilon_P + \frac{1}{1 - \gamma}\epsilon_R \right) \\
    =& \delta \left (\frac{\gamma R_{\text{max}}}{(1 - \gamma)^2}\epsilon_P + \frac{1}{1 - \gamma}\epsilon_R \right).
\end{align}
\end{proof}

\paragraph{Tightness of the bound: } In this paragraph, we present an NS-MDP where \eqref{eqn:apx:lipbound} holds with exact equality, illustrating that the bound given by \thref{thm:lipbound} is tight.

Consider the NS-MDP given in Figure \ref{fig:NSMDPbound}.
Let $\gamma = 0$ and let $\mathcal A = \{a\}$ such that the size of action set $|\mathcal A| = 1$.
%
%
Let the state set $\mathcal S = \{s_1, s_2\}$ and let the initial state for an episode always be state $s_1$.
Let rewards be in the range $[-1, +1]$ such that $R_\text{max} = 1$.
\begin{figure}[h]
    \centering
    \includegraphics[width=0.6\textwidth]{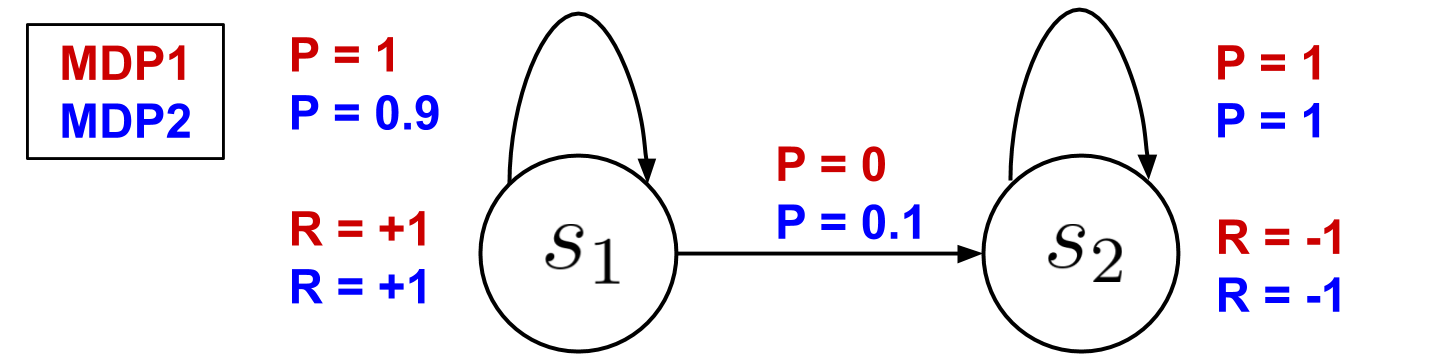}
    \caption{Example NS-MDP.}
    \label{fig:NSMDPbound}
\end{figure}

Notice that for NS-MDP in Figure \ref{fig:NSMDPbound}, $\epsilon_R = |\mathbb{E}[\mathcal R_1(s_1, a)] - \mathbb{E}[\mathcal R_2(s_1, a)]| = 0.2$ as
\begin{align}
    R_1(s_1, a) &= \mathbb{E}[\mathcal R_1(s_1, a)] = 1 \cdot (+1) + 0 \cdot (-1) = 1,
    \\
    R_2(s_1, a) &= \mathbb{E}[\mathcal R_2(s_1, a)] = 0.9 \cdot (+1) + 0.1 \cdot (-1) = 0.8.
\end{align}
Similarly, 
\begin{align}
    \epsilon_P= |\mathcal P_1(s_1|s_1, a) - \mathcal P_2(s_1|s_1, a)| +  |\mathcal P_1(s_2|s_1, a) - \mathcal P_2(s_2|s_1, a)| = 0.2.
\end{align}
Therefore, substituting the values $\gamma=0, R_\text{max}=1, \epsilon_P = \epsilon_R=0.2$, and $\delta=1$ in \eqref{eqn:apx:lipbound}, we get
\begin{align}
    \lvert \rho(\pi, M_{1}) - \rho(\pi, M_{2}) \rvert    &\leq 0.2. \label{apx:eqn:examplebound}
\end{align}

Now to illustrate that the bound is tight, we compute the true difference in performances of a policy $\pi$ for the MDPs given in Figure \ref{fig:NSMDPbound}, i.e., the LHS of \eqref{apx:eqn:examplebound}.
Notice that
%
\begin{align}
    \rho(\pi, M_1) &= (1-\gamma)^{-1}\sum_{s \in \mathcal S}d^\pi(s, M_1) \sum_{a \in \mathcal A} \pi(a|s) R_1(s, a) \overset{\textbf{(a)}}{=} R_1(s_1, a) = 1,
\end{align}
where \textbf{(a)} follows because (i) $\gamma=0$, (ii)  as there is only a single action, $\pi(a|s) = 1$, and (iii) since $s_1$ is the starting state and $\gamma=0$, therefore, $d^\pi(s_1, M_1)=1$ and $d^\pi(s_2, M_1)=0$. 
Similarly, $\rho(\pi, M_2) = R_2(s_1, a) = 0.8$.
Therefore, $|\rho(\pi, M_1) - \rho(\pi, M_2)| = 0.2$, which is exactly equal to the value of bound in \eqref{apx:eqn:examplebound}.

\section{Uncertainty Estimation}
Our \thref{thm:consistentcoverage} makes use of a property proven by  \citet{djogbenou2019asymptotic}.
This property by \citet{djogbenou2019asymptotic} was established for inference about the \textit{parameters} of a regression model.
We leverage this property to obtain confidence intervals for \textit{predictions} of future performance.
In the following section, we first review their results and then in the section thereafter we present the proof of \thref{thm:consistentcoverage}.

\subsection{Preliminary}
\label{apx:sec:prelim}
Before moving forward, we first revisit all the necessary notations and review the result by \citet{djogbenou2019asymptotic}.
For a regression problem, let $Y \in \mathbb{R}^{k\times 1}$ be the stochastic observations, let $\Phi \in \mathbb{R}^{k \times d}$ be the deterministic predicates, and let $w \in \mathbb{R}^{d\times 1}$ be the regression parameters.
Let $\xi \in \mathbb{R}^{k \times 1}$ be a vector of $k$ independent noises.  
The linear system of equations for regression is then given by,
\begin{align}
    Y = \Phi w + \xi. \label{apx:eqn:regression}
\end{align}
The least-squares estimate $\hat w$ of $w$ is given by $\hat w \coloneqq (\Phi ^\top \Phi)^{-1}\Phi^\top Y$ and the estimate $\hat Y \coloneqq \Phi \hat w$.
Subsequently, co-variance of the estimate $\hat w$ can be computed as following,
\begin{align}
    V \coloneqq \mathbb{V}(\hat w) =& \mathbb{E} \left[ \left (\hat w - \mathbb{E}[\hat w] \right) \left(\hat w -\mathbb{E} [\hat w] \right)^\top \right]
    \\
    =& \mathbb{E} \left[\left((\Phi^\top \Phi)^{-1}\Phi^\top( Y - \mathbb{E}[Y])\right) \left((\Phi^\top \Phi)^{-1}\Phi^\top( Y - \mathbb{E}[Y]) \right)^\top \right]
    \\
    =&\mathbb{E} \left[ \left ((\Phi^\top \Phi)^{-1}\Phi^\top \xi\right) \left((\Phi^\top \Phi)^{-1}\Phi^\top\xi \right)^\top \right]
    \\
    =&\mathbb{E} \left[  (\Phi^\top \Phi)^{-1}\Phi^\top \xi \xi^\top \Phi (\Phi^\top \Phi)^{-1} \right] 
    \\
    \overset{\textbf{(a)}}{=}&(\Phi^\top \Phi)^{-1}\Phi^\top \mathbb{E} \left[  \xi \xi^\top \right] \Phi (\Phi^\top \Phi)^{-1} 
    \\
    =& (\Phi^\top\Phi)^{-1}\Phi^\top  \Omega \Phi (\Phi^\top\Phi)^{-1}, \label{apx:eqn:var}
\end{align}
where \textbf{(a)} follows from the fact that $\Phi$ is deterministic, and $\Omega$ is the co-variance matrix of the mean-zero and heteroscedastic noises $\xi$.
Notice that as the noises are independent, the off-diagonal terms in $\Omega$ are zero.
However, since the true $\Omega$ is not known, it can be estimated using $\hat \Omega$ which contains the squared errors from the OLS estimate \cite{mackinnon2012inference}.
That is, let $\hat \xi \coloneqq \hat Y - Y$, then $\hat \Omega$ is a diagonal matrix with $\hat \xi^2$ in the diagonal.
Let such an estimator of $\mathbb{V}(w)$ be,
\begin{align}
    \hat V \coloneqq (\Phi^\top\Phi)^{-1}\Phi^\top  \hat \Omega \Phi (\Phi^\top\Phi)^{-1}. \label{apx:eqn:varr}
\end{align}

Let $b^\top w$ be a desired null hypothesis with $b^\top b =1$.
Let $\texttt{t}_b$, the $\texttt{t}$-statistic for testing this hypothesis, and its pseudo-sample  $\texttt{t}_b^*$ obtained using the wild bootstrap procedure with Rademacher variables $\sigma^*$ be (see Section \ref{sec:CI} in the main body for exact steps),
\begin{align}
    \texttt{t}_b = \frac{b^\top(\hat w - w)}{\sqrt{b^\top \hat V b}}, 
    \quad\quad
    \texttt{t}_b^* \coloneqq \frac{b^\top(\hat w^* - \hat w)}{\sqrt{b^\top \hat V^* b}}.
    \label{apx:eqn:ta}
\end{align}
Note that in \eqref{apx:eqn:ta}, the subscript of $b$ is \textit{not} related to percentile of the previously defined $\texttt{t}$-statistic: $\texttt{t}$  in the main paper.
$\texttt{t}_b$ and $\texttt{t}_b^*$  are new variables.

Now we state the result we use from the work by \citet{djogbenou2019asymptotic}.
This result require two main assumptions.
Our presentations of these assumptions are slightly different from the exact statements given by \citet{djogbenou2019asymptotic}.
The differences are (a) we make the assumptions tighter than what is required for their results to hold, and (b) we ignore a third assumption that is related to cluster sizes, as our setting is a special case where the cluster size is equal to $1$. 
We call these assumptions \textit{requirements} to distinguish them from our assumptions.

\begin{reqm}[Independence]
\thlabel{apx:ass:independence}

$\forall i \in [1,k]$, the noise terms $\xi_i$'s are mean-zero, bounded, and independent random variables.

\end{reqm}

\begin{reqm}[Positive Definite]
\thlabel{apx:ass:pd}

$(\Phi^\top \Phi)^{-1}$ is positive-definite and $\exists C_2>0$ such that $\lVert \Phi \rVert_\infty < C_2$.

\end{reqm}


\begin{lemma}[Theorem 3.2 \citet{djogbenou2019asymptotic}]
\thlabel{apx:lemma:t}
Under \thref{apx:ass:independence,apx:ass:pd}, if $\mathbb{E}[ \sigma^{*3}] < \infty$ and if the true value of $w$ is given by \eqref{apx:eqn:regression}, then as $k \rightarrow \infty$,
\begin{align}
    \Pr\left (\sup_{x \in \mathbb{R}} \,\, \left \lvert \Pr(\texttt{t}_b^*<x) - \Pr(\texttt{t}_b<x) \right \lvert > \alpha \right) \rightarrow 0.
\end{align}
\end{lemma}

\subsection{Proof of Coverage Error}

\label{apx:sec:coverageproof}

First, we recall the notations established in the main body, which are required for the proof.
Using similar steps to those in \eqref{apx:eqn:var}, it can be seen that the variance $V_f$ of the estimator $\hat \rho(\pi, k+\delta) \coloneqq \phi(k+\delta) \hat w$ of future performance is
\begin{align}
    V_f = \phi(k+\delta)(\Phi^\top\Phi)^{-1}\Phi^\top  \Omega \Phi (\Phi^\top\Phi)^{-1}\phi(k+\delta)^\top. \label{apx:eqn:vf}
\end{align}
Similar to before, let an estimate $\hat V_f$ of $V_f$ be defined as,
\begin{align}
    \hat V_f = \phi(k+\delta)(\Phi^\top\Phi)^{-1}\Phi^\top  \hat \Omega \Phi (\Phi^\top\Phi)^{-1}\phi(k+\delta)^\top, \label{apx:eqn:vfhat}
\end{align}
where $\hat \Omega$ is the same as in \eqref{apx:eqn:varr}.
Recall from Section \ref{sec:CI} that the sample standard deviation of the $\phi(k+\delta)\hat w$ is $\hat s = \sqrt{\hat V_f} $ and the pseudo standard deviation $\hat s^* \coloneqq \sqrt{\hat V_f^*}$, where the pseudo variables are created using wild bootstrap procedure outlined in Section \ref{sec:CI}.
Similarly, recall that the $\texttt{t}$-statistic and the pseudo $\texttt{t}$-statistic for estimating future performance are  given by
\begin{align}
     \texttt{t} \coloneqq \frac{\hat \rho(\pi, k+\delta) - \rho(\pi, k+\delta)}{\hat s},  \quad\quad 
     \texttt{t}^* \coloneqq \frac{\hat \rho(\pi, k+\delta)^* - \hat \rho(\pi, k+\delta)}{\hat s^*}. \label{apx:eqn:t}
\end{align}

For the purpose of  \thref{thm:consistentcoverage}, we use a Fourier basis of order $d$, which is given by \cite{bloomfield2004fourier}: 
\begin{align}
    \phi(x) \coloneqq \left \{\frac{\sin(2 \pi n x)}{C} \middle | n \in [1,d] \right \} \cup \left \{\frac{\cos(2 \pi n x) }{C} \middle | n \in [1,d] \right \} \cup \left \{\frac{1}{C} \right \}, \label{apx:eqn:fourier}
\end{align}
where $C\coloneqq \sqrt{d + 1}$.

\begin{thm}[Consistent Coverage] Under \thref{ass:smooth,ass:fullsupport}, if the set of trajectories $\{H_i\}_{i=1}^k$ are independent and if $\phi(x)$ is a Fourier basis of order $d$, then as $k \rightarrow \infty$,
\begin{align}
\Pr\left (\rho(\pi, k+\delta) \in  \left [\hat \rho (\pi, k+\delta) - \hat s \texttt{t}^*_{1 - \alpha/2}, \,\, \hat \rho (\pi, k+\delta) - \hat s \texttt{t}^*_{\alpha/2} \right] \right) \rightarrow 1 - \alpha. \label{apx:eqn:thm}
\end{align}

\end{thm}

\begin{proof}
For the purpose of this proof, we will make use \thref{apx:lemma:t}.
Therefore, we first discuss how our method satisfies the requirements for \thref{apx:lemma:t}.

To satisfy \thref{apx:ass:independence}, recall that in the proposed method, the estimates $\{\hat \rho(\pi, i)\}_{i=1}^k$ of past performances are obtained using counter-factual reasoning.
Therefore, satisfying \thref{apx:ass:independence} in our method requires consideration of two sources of noise:
(a) the noise resulting from the inherent stochasticity in the non-stationary MDP sequence, as given in \thref{ass:smooth}, and
(b) the other noise resulting due to our use of importance sampling to estimate past performances $\{\rho(\pi, i)\}_{i=1}^k$, which are subsequently used to obtain the forecast for $\rho(\pi, k+\delta)$.


Notice that the noises $\{\xi_i\}_{i=1}^k$ inherent to the non-stationary MDP are both mean-zero and independent because of \thref{ass:smooth}.
Further, as importance sampling is unbiased and uses independent draws of trajectories $\{H_i\}_{i=1}^k$, the additional noises in the estimates $\{\hat \rho(\pi, i)\}_{i=1}^k$ are also mean-zero and independent.
The boundedness condition of each $\xi_i$ also holds as (a) all episodic returns are bounded, which is because every reward is bounded between $[-R_\text{max}, R_\text{max}]$ and $\gamma <1$, and
(b) following  \thref{ass:fullsupport}, the denominator of importance sampling ratios are lower bounded by $C$.
Therefore, importance weighted returns are upper bounded by a finite constant.
This makes the noise from importance sampling estimates also bounded.
Hence, all the noises in our performance estimates are independent, bounded, and mean zero.

To satisfy \thref{apx:ass:pd}, 
note that as $\Phi^\top \Phi$ is an inner product matrix, it has to be positive semi-definite. 
Further, as the Fourier basis creates linearly independent features, when  $k>d$ (i.e., it has more samples than number of parameters) the matrix will have full column-rank.
Combining these two points it can be seen that $\Phi^\top\Phi$ is a positive-definite matrix and as the eigenvalues of $(\Phi^\top\Phi)^{-1}$ are just the reciprocals of the eigenvalues of $\Phi^\top\Phi$, the matrix $(\Phi^\top\Phi)^{-1}$ is also positive-definite.
Second half of \thref{apx:ass:pd} is trivially satisfied as all the values of $\phi(x)$ are in $[-1/C, 1/C]$.

Finally, note that when $\phi: \mathbb{N} \rightarrow \mathbb{R}^{1 \times d}$ is a Fourier basis  then $\forall x \in \mathbb{R}, \,\,  \phi(x) \phi(x)^\top = 1$.
To see why, notice from \eqref{apx:eqn:fourier} that 
\begin{align}
    \phi(x) \phi(x)^\top =& \sum_{n=1}^d \left(\frac{\sin(2 \pi n x)}{C}\right)^2 + \sum_{n=1}^d \left(\frac{\cos(2 \pi n x)}{C}\right)^2 + 
     \left(\frac{1}{C}\right)^2
     \\
    =& \frac{ \sum_{n=1}^d \left( \sin^2(2 \pi n x) + \cos^2(2 \pi n x)\right) + 1 }{C^2}
    \overset{\textbf{(a)}}{=} \frac{d+1}{C^2} = 1, \label{apx:eqn:f1}
\end{align}
where \textbf{(a)} follows from the trignometric inequality that $\forall x \in \mathbb R \,\, \sin^2(x) + \cos^2(x) = 1$.

Now we are ready for the complete proof.
For brevity, we define $\mathcal C \coloneqq \left [\hat \rho (\pi, k+\delta) - \hat s \texttt{t}^*_{1 - \alpha/2}, \,\, \hat \rho (\pi, k+\delta) - \hat s \texttt{t}^*_{\alpha/2} \right]$, $\rho \coloneqq \rho(\pi, k+\delta)$, and $\hat \rho \coloneqq  \hat \rho(\pi, k+\delta)$, and expand the LHS of \eqref{apx:eqn:thm},

\begin{align}
    \Pr \left (\rho \in \mathcal C \right) =& \Pr \left (\hat \rho - \hat s \texttt{t}^*_{1 - \alpha/2} \leq \rho \leq \hat \rho - \hat s \texttt{t}^*_{\alpha/2} \right)
    \\
    =& \Pr \left (- \hat s \texttt{t}^*_{1 - \alpha/2} \leq \rho - \hat \rho \leq - \hat s \texttt{t}^*_{\alpha/2} \right)
    \\
    =& \Pr \left (\hat s \texttt{t}^*_{1 - \alpha/2} \geq \hat \rho - \rho \geq \hat s \texttt{t}^*_{\alpha/2} \right)
    \\
    =& \Pr \left (\texttt{t}^*_{1 - \alpha/2} \geq \frac{\hat \rho - \rho}{\hat s} \geq \texttt{t}^*_{\alpha/2} \right). 
    \\
    =& \Pr \left (\texttt{t}^*_{1 - \alpha/2} \geq \texttt{t} \geq \texttt{t}^*_{\alpha/2} \right)
    \\
    =& \Pr \left (\texttt{t} \leq \texttt{t}^*_{1 - \alpha/2} \right) - \Pr \left( \texttt{t} \leq \texttt{t}^*_{\alpha/2} \right). \label{apx:eqn:tstar}
\end{align}
To simplify \eqref{apx:eqn:tstar}, let $b = \phi(k+\delta)^\top$.
Under this instantiation of $b$, the null hypothesis $b^\top w$ in \ref{apx:sec:prelim} for our setting corresponds to $\phi(k+\delta) w$,  which is the true future performance under \thref{ass:smooth}.
Further, for this instantiation of $b$, note from \eqref{apx:eqn:f1} that $b^\top b = 1$.
Now, it can be seen from  \eqref{apx:eqn:vfhat} that $\hat V_f = b^\top \hat V b$, and $\hat V_f^* = b^\top \hat V^* b$.
Thus, $\texttt{t} = \texttt{t}_b$ and $\texttt{t}^* = \texttt{t}_b^*$.
Finally, note that as $\sigma^*$ corresponds to the Rademacher random variable, $\mathbb{E}[\sigma^{*3}] = 0$.

Therefore, leveraging \thref{apx:lemma:t}, in the limit, for any $x$, we can substitute $\Pr(\texttt{t}<x)$ with $\Pr(\texttt{t}^*<x)$ in \eqref{apx:eqn:tstar},
This substitution yields.
\begin{align}
    \Pr \left (\rho \in \mathcal C \right) &\rightarrow \Pr \left (\texttt{t}^* \leq \texttt{t}^*_{1 - \alpha/2}\right) - \Pr \left( \texttt{t}^* \leq \texttt{t}^*_{\alpha/2} \right)
    \\
    &= (1 - \alpha/2) - (\alpha/2)
    \\
    %
    %
    &= 1  -\alpha. \qedhere
\end{align}
\end{proof}

Notice that using the Fourier basis, we were able to satisfy the condition that $b^\top b = 1$ directly.
This allowed us to leverage \thref{apx:lemma:t} without much modification.
However, as noted by \citet{djogbenou2019asymptotic}, the constraint on $b^\top b$ is not necessary and was used to simplify the proof.

\section{Extended Discussion on Bootstrap}
\label{apx:sec:bootstrap}

The goal of this section is to provide additional discussion on (wild) bootstrap for completeness.
Therefore, this section contains a summary of existing works and has no original technical contribution.
We begin by first discussing the idea behind any general bootstrap and the wild bootstrap method.
Subsequently, we discuss alternatives to wild bootstrap.

In many practical applications, it is often desirable to infer distributional properties (e.g., CIs) of a desired statistic of data (e.g., mean).
However, in practice, it is often not possible to get multiple estimates of the desired statistic in a data-efficient way.
To address this problem, bootstrap methods have received wide popularity in the field of computational statistics \cite{efron1994introduction}.

 The core principle of any bootstrap procedure is to \textit{re-sample} the observed data-set $\mathcal D$ and construct multiple \textit{pseudo data-sets} $\mathcal D^*$ in a way that closely mimics the original \textit{data generating process} (DGP).
 This allows to create an \textit{empirical distribution} of the desired statistic by leveraging multiple pseudo data-sets $\mathcal D^*$ \citep{efron1994introduction}.
For example, an empirical distribution containing $B$ estimates of the sample mean can be obtained by generating $B$ pseudo data-sets, where each data-set contains $N$ samples uniformly drawn (with replacement) from the original data-set of size $N$.

For an excellent introduction to bootstrap CIs, refer to the works by \citet{efron1994introduction} and \citet{diciccio1996bootstrap}.
The book by \citet{hall2013bootstrap} provides a thorough treatment of these methods using \textit{Edgeworth expansion}, illustrating when and how bootstrap methods can provide significant advantage over other methods. 
For a very readable practitioner's guide that touches upon several important aspects, refer to the work by \citet{carpenter2000bootstrap}.

\subsection{Why does wild bootstrap work?} 
\label{apx:sec:whyboot}

The original idea of wild bootstrap was proposed by \citet{wu1986jackknife} and later developed by \citet{liu1988bootstrap}, \citet{mammen1993bootstrap}, and \citet{davidson1999wild,davidson2008wild}.
The following summary about the wild bootstrap process is based on an excellent tutorial by \citet{mackinnon2012inference}.

Consider the system of equations in \eqref{apx:eqn:regression}.
The key idea of wild-bootstrap is that the uncertainty in regression estimates (of parameters/predictions) is due to the noise $\xi$ in the observations.
Therefore, if the pseudo-data $Y^*$ is generated such that the noise $\xi^*$ in the data generating process for $Y^*$ resembles the properties of the true underlying noise $\xi$,
then with multiple redraws of such $Y^*$ one can obtain an empirical distribution of the desired statistic (which for our case, corresponds to the forecast of a policy $\pi$'s performance).
This can then be used to estimate the CIs.


As true noise $\xi$ is unobserved, it raises a question about how to estimate its properties to generate $Y^*$.
Fortunately, as ordinary least-squares is an unbiased estimator of parameters/predictions \cite{wasserman2013all}, regression errors $\hat \xi$ can be used as a substitute for the true noise.    
Therefore, to mimic the underlying data generating process, it would be ideal to have bootstrap error terms $\xi^*$ that have similar moments as $\hat \xi$.
%
%
Following the work by \citet{davidson1999wild}, we set $Y^* \coloneqq \hat Y + \xi^*$, where $\xi^* \coloneqq \hat \xi \odot \sigma^*$, and $\sigma^* \in \mathbb{R}^{k \times 1}$ is the independent Rademacher random variable (i.e., $\forall i \in [1,k], \,\, \Pr(\sigma_i^*=+1) = \Pr(\sigma_i^*=-1) = 0.5$).
This choice of $\sigma_i^*$, for all $i \in [1,k]$, ensures that $\xi_i^*$ has the desired zero mean and the same higher-order \textit{even} moments as $\hat \xi_i$ because,
\begin{align}
    \forall i, \,\, \mathbb{E}[\sigma_i^*]=0, \,\, \mathbb{E}[{\sigma_i^*}^2]=1, \,\, \mathbb{E}[{\sigma_i^*}^3]=0, \,\, \mathbb{E}[{\sigma_i^*}^4]=1.
    \label{eqn:Rademacher}
\end{align}

Therefore, for the purpose of this paper, pseudo performances $Y^*$ generated using pseudo-noise $\xi^*$ allow generating a distribution of $\hat \rho(\pi, k+\delta)^*$ that closely mimics the distribution of forecasts $\hat \rho(\pi, k+\delta)$ that would have been generated if we had the true underlying data generating process.


 \subsection{Why not use other bootstrap methods?} 
 \label{apx:sec:otherboot}
 One popular non-parametric technique for bootstrapping in regression is to re-sample, with replacement, $(x,y)$ pairs from the set of observed samples $(X,Y)$ \cite{carpenter2000bootstrap}.
However, in our setup, $X$ variable corresponds to the (deterministic) time index and thus there exists no joint distribution between the $X$ and the $Y$ variables from where time can be sampled stochastically. 
Therefore, paired re-sampling will not mimic the underlying data generative process in our setting.

A semi-parametric technique overcomes the above problem by only re-sampling $Y$ variable as follows.
First, a model is fit to the observed data $(X,Y)$ and predictions $\hat Y$ are obtained.
Then an empirical cumulative distribution function, $\Psi(e)$ of all the errors, $e \coloneqq Y - \hat Y$, is obtained. 
Subsequently, new bootstrapped variables are created as $Y^* \coloneqq \hat Y + \xi^*$, where $\xi^*$ is the re-sampled noise from $\Psi(e)$ \cite{efron1994introduction}.
However, such a process assumes that noises are homoscedastic, which will be severely violated for our purpose.

Another popular technique for \textit{auto-correlated} data uses the idea of \textit{block re-sampling} \cite{efron1994introduction}.
However, this assumes that the underlying process is stationary, and hence is not suitable for our purpose.

\subsection{Why not use standard \texttt{t}-test?}
\label{apx:sec:ttest}
Standard \texttt{t}-test assumes that the predictions will follow the student-\texttt{t} distribution.
Such an assumption can be severely violated, specially in the presence of heteroscedasticity, and heavy tailed noises, when the sample size is not sufficiently large.
Unfortunately, in our setting, use of multiple behavior policies result in heteroscedasticity and importance sampling results in heavy tailed distribution \cite{thomas2015higheval} for counterfactual estimates of past performances. 
It can also be shown that for a finite sample of size $n$, the coverage error of CIs obtained using standard \texttt{t}-statistic is of order $O(n^{-1/2})$ \cite{wasserman2013all,hall2013bootstrap}.
In comparison, it can be shown using Edgeworth expansions \cite{hall2013bootstrap} that the coverage error rate of CIs obtained using bootstrap methods typically provide higher-order refinement by providing error rates up to $O(n^{-p/2})$, where $p \in [1,3]$ \cite{hall1989unusual, diciccio1996bootstrap, hall2013bootstrap}.
%
%
For more elaborate discussions in the context of wild bootstrap, see the work by \citet{kline2012higher} and by \citet{djogbenou2019asymptotic}.
Also, see the work by \citet{mammen1993bootstrap} for detailed empirical comparison of standard t-test against wild-bootstrap.

\section{Algorithm}
\label{sec:apx:algorithm}
In Algorithms \ref{Alg:2}-\ref{Alg:1},\footnote{When $(\alpha/2) B$  or $(1-\alpha/2)B$ is not an integer, then $\texttt{floor}$ or $\texttt{ceil}$ operation should be used, respectively.} we provide the steps for our method: SPIN.
In Algorithm \ref{Alg:3}, PDIS is shorthand for per-decision importance sampling discussed in Section \ref{sec:CI}.
In the following, we discuss certain aspects of SPIN, especially pertaining to the search of a candidate policy $\pi_c$.

\begin{minipage}[t]{0.5\textwidth}
	\IncMargin{1em}
	\begin{algorithm2e}[H]
		\textbf{Input} Predicates $\Phi$, Targets $Y$, Forecast time(s) $\tau$
		\\
        $H \leftarrow (\Phi^\top \Phi)^{-1}\Phi^\top$
        \\
        $\varphi \leftarrow [\phi(\tau_1),...,\phi(\tau_\delta)]$
        \\
        $\hat Y \leftarrow \Phi HY$
        \\
        $\hat \rho \leftarrow \texttt{mean}( \varphi HY)$
        \\
        $\hat \xi \leftarrow  Y - \hat Y$
        \\
        $\hat \Omega \leftarrow \texttt{diag}(\hat \xi^2)$
        \\
        $\hat V \leftarrow \texttt{mean}(\varphi H \hat \Omega H^\top \varphi^\top)$
        \\
        Return $\hat \rho, \hat V, \hat \xi$
		\caption{Forecast}
		\label{Alg:2}  
	\end{algorithm2e}
	\DecMargin{1em} 
	\IncMargin{1em}
	\begin{algorithm2e}[H]
		\textbf{Input} Data $\mathcal D$, Policy $\pi$, Safety-violation rate $\alpha$, Forecast time(s) $\tau$
		\\
        $\Phi \leftarrow \emptyset, Y \leftarrow \emptyset$
        \\
        
		\vspace{8pt}
		\nonl \textcolor[rgb]{0.5,0.5,0.5}{\# Create regression variables}
		\\
        \For{$(k, h) \in \mathcal D$}
        {
            $ \hat \rho(\pi, k) \leftarrow \text{PDIS}(\pi, h)$
            \\
            $\Phi.\texttt{append}(\phi(k))$
            \\
            $Y.\texttt{append}(\hat \rho(\pi, k))$
        }

		\vspace{8pt}
		$\hat \rho, \hat V, \hat \xi \leftarrow \text{Forecast}(\Phi, Y, \tau)$

		\vspace{8pt}
		\nonl \textcolor[rgb]{0.5,0.5,0.5}{\# Wild Bootstrap (in parallel)}
		\\
		$\texttt{t}^* \leftarrow \emptyset, \texttt{t}^{**} \leftarrow \emptyset$
		\\
		\For{$i \in [1,...,B]$}
		{
		    $\sigma^* \leftarrow [\pm 1, \pm 1, ..., \pm 1]$
		    \\
        	    $\xi^* \leftarrow \hat \xi \odot \sigma^*$
    	    \\
    	    $ Y^* \leftarrow \hat Y + \xi^*$
    	    \\
    		$\hat \rho^*, \hat V^*, \_ \leftarrow \text{Forecast}(\Phi, Y^*, \tau) $
    		
    	    $\texttt{t}^*[i] \leftarrow  (\hat \rho^* - \hat \rho)/ \sqrt{\hat V^*} $
    	}

		\vspace{8pt}
		\nonl \textcolor[rgb]{0.5,0.5,0.5}{\# Get prediction interval}
		\\
		$\texttt{t}^{**} \leftarrow \texttt{sort}(\texttt{t}^*)$
		\\
		$\hat \rho^{\text{lb}} \leftarrow \hat \rho -  \texttt{t}^{**}[ (1-\alpha/2) B] \sqrt{\hat V}$
		\\
		$\hat \rho^{\text{ub}} \leftarrow \hat \rho - \texttt{t}^{**}[(\alpha/2) B] \sqrt{\hat V}$
        \\
        
        \vspace{8pt}
        Return ($\hat \rho^{\text{lb}}, \hat \rho^{\text{ub}}$)
		\caption{PI: Prediction Interval}
		\label{Alg:3}  
	\end{algorithm2e}
	\DecMargin{1em} 
\end{minipage}
\begin{minipage}[t]{0.5\textwidth}
	\IncMargin{1em}
	\begin{algorithm2e}[H]
		\textbf{Input} Safety-violation rate $\alpha$, Initial safe policy $\pi^\text{safe}$, Entropy-regularizer  $\lambda$, Batch-size $\delta$\\
		\textbf{Initialize} $\mathcal{D}_\text{train} \leftarrow \emptyset$,
		$\mathcal{D}_\text{test}\leftarrow \emptyset$, \,\,
		$\pi \leftarrow \pi_1^\text{safe}$, \,\, $k \leftarrow 0$.
		\\
		\While{True}
		{
			\nonl \textcolor[rgb]{0.5,0.5,0.5}{\# Collect new trajectories using $\pi$}
			\\
			$ \mathcal D \leftarrow \emptyset $
			\\
			\For {$\texttt{episode} \in [1,2,..., \delta]$}
			{
			    $k \leftarrow k +1$
    			\\
        		$h \leftarrow \{(s^t, a^t, \Pr(a^t|s^t), r^t) \}_{t=0}^T$
        		\\
        		$ \mathcal D \leftarrow  \mathcal D \cup (k, h)$
    		}
    		\vspace{8pt}
    		\nonl
    		\nonl \textcolor[rgb]{0.5,0.5,0.5}{\# Split data}
    		\\
    		$\mathcal D_1, \mathcal D_2 \leftarrow \texttt{split}(\mathcal D)$
    		\\
    		$\mathcal D_{\text{train}} \leftarrow \mathcal D_1 \cup \mathcal D_{\text{train}} $
    		\\
    		$\mathcal D_{\text{test}} \leftarrow \mathcal D_2 \cup D_{\text{test}}$
    		
    		\vspace{8pt}
    		\nonl
    		\textcolor[rgb]{0.5,0.5,0.5}{\# Candidate search} 
    		\\
    		$\tau \leftarrow [k+1,...,k+\delta]$
    		\\
			$\hat \rho^\text{lb}(\pi), \_ \leftarrow \text{PI}(\mathcal D_{\text{train}}, \pi, \alpha/2, \tau)$
    		\\
    		$\pi_c \leftarrow \text{argmax}_{\pi}  [ \hat \rho^\text{lb}(\pi) + \lambda \mathcal H(\pi, \mathcal D_{\text{train}})] $
    		\\
    		\vspace{8pt}
    		\nonl 
    		\textcolor[rgb]{0.5,0.5,0.5}{\# Safety test}\\
    		
    		$\hat \rho^\text{lb}, \_ \leftarrow  \text{PI}(\mathcal D_\text{test}, \pi_c,\alpha/2, \tau)$
    		\\
    		$\_, \hat \rho^\text{ub} \leftarrow  \text{PI}(\mathcal D_\text{test}, \pi^\text{safe}, \alpha/2, \tau)$\\
    		\vspace{5 pt}
    		\If{$\hat \rho^\text{lb} > \hat \rho^\text{ub}$}
    		{
        		$\pi \leftarrow \pi_c$
    		}
    		\Else
    		{
        		$\pi \leftarrow \pi^\text{safe}$
    		}
        }
		\caption{SPIN: Safe Policy Improvement for Non-stationary settings}
		\label{Alg:1}  
	\end{algorithm2e}
	\DecMargin{1em} 
\end{minipage}

\textbf{Mean future performance: }
In many practical applications, it is often desirable to reduce computational costs by executing a given policy $\pi$ for multiple episodes before an update, i.e., $\delta > 1$.
This raises the question regarding which episode, among the future $\delta$ episodes, should a policy $\pi$ be optimized for before execution?
To address this question, in settings where $\delta>1$, instead of choosing a single future episode's performance for optimization and safety check, we propose using the average performance across all the $\delta$ future episodes, i.e., $(1/\delta)\sum_{i=1}^\delta\rho(\pi, k+i)$.

\paragraph{Differentiating the lower bound:}
SPIN proposes a candidate policy $\pi_c$ by finding a policy $\pi$ that maximizes the lower bound $\hat \rho^\text{lb}$ of the future performance (Line 14 in Algorithm \ref{Alg:1}).
To find $\pi_c$ efficiently, we propose using a differentiable optimization procedure.
A visual illustration of the process is given in Figure \ref{fig:diffLB}.

Derivatives of most of the steps in Algorithms \ref{Alg:2}-\ref{Alg:3} can directly be taken care by modern automatic differentiable programming libraries.
Hence, in the following, we restrict the focus of our discussion for describing a \textit{straight-through} gradient estimator for sorting  performed in Line 15 in Algorithm \ref{Alg:3}.
Note that sorting is required to obtain the \textit{ordered-statistics} to create an empirical distribution of $\texttt{t}^*$ such that in Line 16 and 17 of Algorithm \ref{Alg:3} the desired percentiles of $\texttt{t}^*$ can be obtained. 

We first introduce some notations.
Let $\texttt{t}^* \in \mathbb{R}^{B \times 1}$ be the unsorted array and $\texttt{t}^{**} \in \mathbb{R}^{B \times 1}$ be its sorted counterpart.
To avoid breaking ties when sorting, we assume that there exists $C_3 > 0$ such that all the values of $\texttt{t}^*$ are separated by at least $C_3$.
Let $\Gamma \in \{0,1\}^{B \times B}$ be a \textit{permutation matrix} (i.e., $\forall (i,j), \,\, \Gamma(i,j) \in \{0,1\}$, and each row and each column of $\Gamma$ sums to $1$)
obtained using any sorting function such that $\texttt{t}^{**} = \Gamma \texttt{t}^*$.
%
%
This operation has a computational graph as shown in Figure \ref{fig:diffSort}.

\begin{wrapfigure}{R}{0.35\textwidth}
    \centering
    \includegraphics[width=0.35\textwidth]{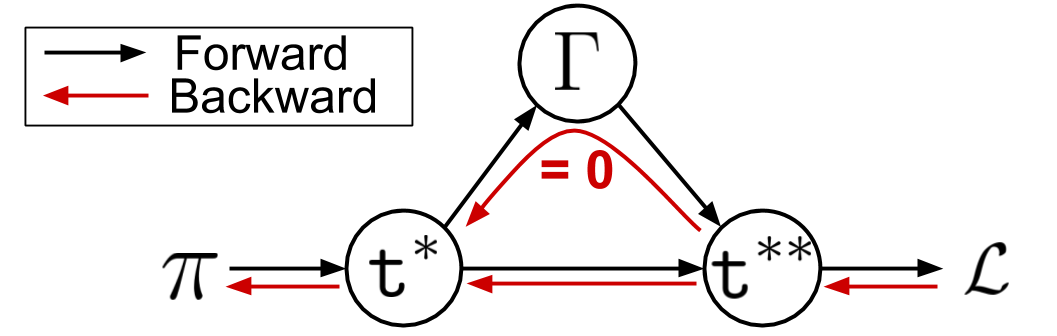}
    \caption{Computational graph for obtaining ordered-statistics $\texttt{t}^{**}$.}
    \label{fig:diffSort}
\end{wrapfigure}

Notice that when the values to be sorted are perturbed by a very small amount, the order of the sorted array remains the same (e.g., sorting both the array $[30,10,20]$ and its perturbed version results in $[10,20,30]$).
%
%
%
That is, if $\texttt{t}^*$  is perturbed by an $\epsilon \rightarrow 0$, then the  $\Gamma$ obtained using the sorting function will not change at all. 
Therefore, the derivative of $\Gamma$ with respect to $\texttt{t}^*$ is $0$ and derivative of a desired loss function $\mathcal L$ with respect to $\texttt{t}^*$ is
\[\frac{\partial \mathcal L}{\partial \texttt{t}^*} = \Gamma^{\top}\frac{\partial \mathcal L}{\partial \texttt{t}^{**}} = \Gamma^{-1}\frac{\partial \mathcal L}{\partial \texttt{t}^{**}},\]
as for any permutation matrix, $\Gamma^\top = \Gamma^{-1}$. 
Therefore, derivatives are back-propagated through the sorting operation in a straight-through manner by directly performing \textit{un}-sorting.

More advanced techniques for differentiable sorting has been proposed by \citet{cuturi2019differentiable} and \citet{blondel2020fast}.
These methods can be leveraged to further improve our algorithm.
We leave these for future work.

\paragraph{Entropy regularization: }

As we perform iterative safe policy improvement, the current policy $\pi$ becomes the behavior policy $\beta$ for future updates.
Therefore, if the current policy $\pi$ becomes nearly deterministic then the\textit{ past performance estimates for a future policy}, which is computed using importance sampling, can suffer from high-variance.
To mitigate this issue, we add a $\lambda$ regularized entropy bonus $\mathcal H$ in the optimization objective.
This is only done during candidate policy search and hence does not impact the safety check procedure.

\paragraph{Percentile CIs: }
Notice that each step of the inner optimization process to search for a candidate policy $\pi_c$ requires computing multiple estimates of the pseudo standard deviation $\hat s^*$, one for each sample of $\texttt{t}^*$, using wild-bootstrap to obtain the CIs.
This can be computationally expensive for real-world applications that run on low-powered devices.
As an alternative, we propose using the \textit{percentile} method \cite{carpenter2000bootstrap,efron1994introduction} during the candidate policy search, which unlike the $\texttt{t}$-statistic method does not require computing $\hat s^*$. 

While percentile method can offer significant computational speed-up, the CIs obtained from it are typically less accurate than those obtained from the method that uses the $\texttt{t}$-statistic \cite{carpenter2000bootstrap,efron1994introduction}.
To get the best of both, (i) as  searching for $\pi_c$ requires an inner optimization routine and accuracy of CIs are less important, therefore we use the percentile method, and (ii) as the safety test requires no inner optimization and accuracy of CIs are more important to ensure safety, we use the $\texttt{t}$-statistic method.

To obtain the CIs on $\rho(\pi, k+\delta)$ using the percentile method, let $\Psi$ denote the empirical cumulative distribution function (CDF) of the pseudo performance forecasts $\hat \rho(\pi, k+\delta)^*$.
Then a $(1-\alpha)100\%$ CI, $[\hat \rho^\text{lb}, \hat \rho^\text{ub}]$, can be estimated as $[\Psi^{-1}(\alpha/2), \Psi^{-1}(1 - \alpha/2)]$, where $\Psi^{-1}$ denotes the inverse CDF distribution.
That is, if $\rho^*$ is an array of $B$ pseudo samples of $\hat \rho(\pi, k+\delta)$, and $\rho^{**}$ contains its sorted ordered-statistics, then a $(1-\alpha)100\%$ CI for $\rho(\pi, k+\delta)$ is $[\rho^{**}[(\alpha/2)B], \rho^{**}[(1-\alpha/2)B]]$.
Gradients of the lower bound from the percentile method can be computed using the same straight-through gradient estimator discussed earlier.

\paragraph{Complexity analysis (space, time, and sample size): }
Memory requirement for SPIN is linear in the number of past episodes as it stores all the past data to analyze performance trend of policies.  
As both SPIN and Baseline \cite{thomas2015high,thomas2019preventing} incorporate an inner optimization loop, the computational cost to search for a candidate policy $\pi_c$ before performing a safety test is similar. 
Additional computational cost is incurred by our method as it requires computing $(\Phi^\top\Phi)^{-1}$ and $\hat V$ in Algorithm \ref{Alg:2} for time series analysis.
%
However, note that as $\Phi^\top\Phi \in \mathbb{R}^{d\times d}$, where $d$ is the dimension of basis function and $d << k$, the cost of inverting $\Phi^\top\Phi$ is negligible.
To avoid the computational cost of computing $\hat V$, the percentile method can be used during candidate policy search (as discussed earlier), and the $\texttt{t}$-statistic method can be used only during the safety test to avoid compromising on safety.
%
%
%
An empirical comparison of sample efficiency of SPIN and Baseline is presented in Figure \ref{fig:results}.

\section{Extended Empirical Details}
\label{apx:sec:experiments}

\subsection{Domains}
\label{apx:sec:domains}

%

\textbf{Non-stationary Recommender System (RecoSys): }
Online recommendation systems for tutorials, movies, advertisements and other products are ubiquitous \citep{theocharous2015personalized,theocharous2020reinforcement}.
Personalizing for each user is challenging in such settings as interests of an user for different items among the products that can be recommended fluctuate over time.
For an example, in the context of online shopping, interests of customers can vary based on seasonality or other unknown factors.
To abstract such settings, in this domain the reward (interest of the user) associated with each item changes over time.
%
%

For $\pi^\text{safe}$, we set the probability of choosing each item proportional to the reward associated with each item in MDP $M_1$.
This resembles how recommendations would have got set by an expert system initially, such that most relevant recommendation is prioritized while some exploration for other items is also ensured.

\textbf{Non-stationary Diabetes Treatment: } 
This NS-MDP is modeled using an open-source implementation \citep{simglucose} of the U.S. Food and Drug Administration (FDA) approved Type-1 Diabetes Mellitus simulator (T1DMS) \citep{man2014uva} for the treatment of Type-1 diabetes,
where we induced non-stationarity by oscillating the body parameters (e.g., rate of glucose absorption, insulin sensitivity, etc.) between two known configurations available in the simulator. 
Each step of an episode corresponds to a minute in an \textit{in-silico} patient's body and is governed by a continuous time non-linear ordinary differential equation (ODE) \citep{man2014uva}.

Notice that as the parameters that are being oscillated are inputs to a non-linear ODE system, the exact trend of performance for any policy in this NS-MDP is unknown.
This more closely reflects a real-world setting where \thref{ass:smooth} might not hold, as every policy's performance trend in real-world problems cannot be expected to follow \textit{any} specific trend \textit{exactly}--one can only hope to obtain a coarse approximation of the trend.

%

To control the insulin injection, which is required for regulating the blood glucose level, we use a parameterized policy based on the amount of insulin that a person with diabetes is instructed to inject prior to eating a meal \citep{bastani2014model}:
\begin{align}
    \text{injection} = \frac{\text{current blood glucose} - \text{target blood glucose}}{CF} + \frac{\text{meal size}}{CR},
\end{align}
where `current blood glucose' is the estimate of the person's current blood glucose level, `target blood glucose' is the desired blood glucose, `meal size' is the estimate of the size of the meal the patient is about to eat, and $CR$ and $CF$ are two real-valued parameters that must be tuned based on the body parameters to make the treatment effective.
We set $\pi^\text{safe}$ to a value near the optimal $CR$ and $CF$ values for MDP $M_1$.
This resembles how the values would have got set during a patient's initial visit to a medical practitioner.

\subsection{Baseline}
\label{apx:sec:baseline}

For a fair comparison, Baseline used for our experiments corresponds to the algorithm presented by \citet{thomas2015high}, which is also a type of Seldonian algorithm \cite{thomas2019preventing}.
While this algorithm is also designed to ensure safe policy improvement, it assumes that the MDP is stationary.
Specifically, during the safety test it ensures that a candidate policy's performance is higher than that of $\pi^\text{safe}$'s by computing CIs on the \textit{average} performance over the past episodes.

\subsection{Hyper-parameters}
\label{apx:sec:hyper}
In Table \ref{tab:hyper}, we provide hyper-parameter (HP) ranges that were used for SPIN and Baseline for both the domains.
As obtaining optimal HPs is often not feasible in practical scenarios, algorithms that ensure safety should be robust to how an end-user sets the HPs.
Therefore, we set the hyper-parameters within reasonable ranges and report the results in Figure \ref{fig:results}.
These results are aggregated over the \textit{entire} distribution of hyper-parameters, and \textit{not} just for the best hyper-parameter setting.
    This choice is motivated by the fact that best performances can often be misleading as it only shows what an algorithm \textit{can} achieve and not what it is \textit{likely} to achieve \cite{jordan2018using}.

    For both RecoSys and Diabetes, we ran $1000$ HPs per algorithm, per speed, per domain. 
    For RecoSys, we ran $10$ trials per HP and $1$ trial per HP for diabetes treatment as it involves solving a continuous time ODE and hence is relatively  computationally expensive.
    For experiments, the authors had shared access to a computing cluster, consisting of 50 compute nodes with 28 cores each.

\begin{table}[h]
    \centering
    \begin{tabular}{ccc}
        \hline    \textbf{Algorithm} & \textbf{Hyper-parameter} & \textbf{Range} \\
        \hline \\
        SPIN \& Baseline & $\alpha$ & $0.05$\\
        SPIN \& Baseline & $\delta$ & $\{2,4,6,8\}$\\
        SPIN \& Baseline & $N$ & $\delta \times$ \texttt{uniform}(\{$2,5$\})\\
        SPIN \& Baseline & $\eta$ & $10^{-1}$ \\
        SPIN \& Baseline & $\lambda$ (RecoSys) & $\texttt{loguniform}(5 \times 10^{-5}, 10^0$)\\
        SPIN \& Baseline & $\lambda$ (Diabetes) & $\texttt{loguniform}(10^{-2}, 10^0$)\\
        SPIN \& Baseline & $B$ (candidate policy search) & $200$\\
        SPIN \& Baseline & $B$ (safety test) & $500$\\
        SPIN & $d$  & $\texttt{uniform}(\{2,3,4,5\})$\\
    \end{tabular}
    \caption{Here, $N$ and $\eta$ represents the number of gradient steps, and the learning rate used while performing Line 14 of Algorithm \ref{Alg:1}. The dimension of Fourier basis is given by $d$. Notice that $d$ is set to different values to provide results for different settings where SPIN is \textit{incapable} of modeling the performance trend of policies exactly, and thus \thref{ass:smooth} is violated. This resembles practical settings, where it is not possible to exactly know the true underlying trend--it can only be coarsely approximated.}
    \label{tab:hyper}
\end{table}

\subsection{Plot Details}
\label{apx:sec:plots}

In Figure \ref{fig:results}, the plot on the bottom-left corresponds to how often unsafe policies were executed during the process whose learning curves were plotted in Figure \ref{fig:results} (top-left).
It can be seen that SPIN remains safe almost always.
The middle and the right plots in the top row of Figure \ref{fig:results} show the normalized performance improvement over the known safe policy $\pi^\text{safe}$.   
The middle and the right plots in the bottom row of Figure \ref{fig:results} show how often unsafe policies were executed.

Note that the performance for any policy $\pi$ is defined in terms of the expected return.
However, for the diabetes domain, we do not know the exact performances of any policy--we can only observe the returns obtained.
Therefore, even when an $\texttt{alg}$ selects $\pi^\text{safe}$, it is not possible to get an accurate estimate of its safety violation rate by directly averaging returns observed using a finite number of trials.
To make the evaluation process more accurate, we use the following evaluation procedure.

Let a policy $\pi$ be `unsafe' when $\rho(\pi, k+\delta) < \rho(\pi^\text{safe}, k+\delta)$, and let $\pi_c$ denote policies not equal to $\pi^\text{safe}$, then,
\begin{align}
    \Pr(\texttt{alg}(\mathcal D) = \text{unsafe}) &= \Pr( \pi_c = \text{unsafe}|\texttt{alg}(\mathcal D) = \pi_c)\Pr(\texttt{alg}(\mathcal D) = \pi_c) \\
    &\quad + \Pr( \pi^\text{safe} = \text{unsafe}|\texttt{alg}(\mathcal D) = \pi^\text{safe})\Pr(\texttt{alg}(\mathcal D) = \pi^\text{safe}) 
    \\
    &\overset{\textbf{(a)}}{=} \Pr( \pi_c = \text{unsafe}|\texttt{alg}(\mathcal D) = \pi_c)\Pr(\texttt{alg}(\mathcal D) = \pi_c),
\end{align}
where \textbf{(a)} holds because $\Pr( \pi^\text{safe} = \text{unsafe}) = 0$.
%
%
Therefore, to evaluate whether $\texttt{alg}(\mathcal D)$ is unsafe, for each episode we compare the sample average of returns obtained whenever $\texttt{alg}(\mathcal D) \neq \pi^\text{safe}$ to the sample average of returns observed using $\pi^\text{safe}$, multiplied by the probability of how often $\texttt{alg}(\mathcal D) \neq \pi^\text{safe}$ .

\end{document}